%% file: main.tex
\documentclass[11pt]{article}
\usepackage{tgpagella}
\usepackage[top=1in, bottom=1in, left=1in, right=1in]{geometry}
\input{macros/macro}

\usepackage{cite}

\title{The Law of Parsimony in Gradient Descent for Learning Deep Linear Networks}

\author[1]{Can Yaras\footnote{The first two authors contributed to this work equally.}}
\author[1]{Peng Wang\protect\CoAuthorMark}
\author[1]{Wei Hu}
\author[2]{Zhihui Zhu}
\author[1]{Laura Balzano}
\author[1]{Qing Qu}

\newcommand\CoAuthorMark{\footnotemark[\arabic{footnote}]}
\affil[1]{Department of Electrical Engineering \& Computer Science, University of Michigan}
\affil[2]{Department of Computer Science \& Engineering, Ohio State University}

\begin{document}

\maketitle

\begin{abstract}
Over the past few years, an extensively studied phenomenon in training deep networks is the implicit bias of gradient descent towards parsimonious solutions. In this work, we investigate this phenomenon by narrowing our focus to deep linear networks. Through our analysis, we reveal a surprising ``law of parsimony'' in the learning dynamics when the data possesses low-dimensional structures. Specifically, we show that the evolution of gradient descent starting from orthogonal initialization only affects a minimal portion of singular vector spaces across all weight matrices. In other words, the learning process happens only within a small invariant subspace of each weight matrix, despite the fact that all weight parameters are updated throughout training. This simplicity in learning dynamics could have significant implications for both efficient training and a better understanding of deep networks. First, the analysis enables us to considerably improve training efficiency by taking advantage of the low-dimensional structure in learning dynamics. We can construct smaller, equivalent deep linear networks without sacrificing the benefits associated with the wider counterparts. Second, it allows us to better understand deep representation learning by elucidating the linear progressive separation and concentration of representations from shallow to deep layers. We also conduct numerical experiments to support our theoretical results. The code for our experiments can be found at \url{https://github.com/cjyaras/lawofparsimony}.
\end{abstract}

\section{Introduction}\label{sec:intro}
\input{main_sections/intro}

\section{Problem Formulation}\label{sec:problem}
\input{main_sections/problem}

\section{A Law of Parsimony in Gradient Descent of DLNs}\label{sec:results}
\input{main_sections/results}

\section{Applications and Experiments}\label{sec:appli}

In this section, we focus on showcasing two specific applications of \Cref{thm:main}. In \Cref{sec:low-rank}, we present the utilization of \textbf{Case 1} in \Cref{thm:main} to improve the speed of training in low-rank deep matrix completion. In \Cref{sec:separation}, we apply the findings from \textbf{Case 2} in \Cref{thm:main} to gain insights into the phenomenon of progressive feature collapse in representation learning.

\subsection{Application I: Accelerating Deep Low-Rank Matrix Completion}\label{sec:low-rank}
\input{main_sections/application1}
\subsection{Application II: Understanding Progressive Feature Collapse in DLN}\label{sec:separation}
\input{main_sections/application2}

\section{Conclusion}\label{sec:conclusion}
\input{main_sections/conclusion}

{\small 
\bibliographystyle{unsrt}
\bibliography{biblio/refs, biblio/DNN, biblio/neural_collapse}
}

\newpage

\appendix

\input{main_sections/app_main}

\end{document}

%% file: macros/macro.tex
\usepackage{url}
\usepackage[hidelinks]{hyperref}
\usepackage{amsmath,amsthm,amssymb,amsbsy,amsfonts,amscd,bm}
\usepackage{paralist}
\usepackage{xcolor}
\usepackage{color}
\usepackage{graphicx}
\graphicspath{{./figs/}}
\usepackage{algorithm}
\usepackage{algorithmic}
\usepackage{comment}
\usepackage{multirow}
\usepackage{enumitem}
\usepackage{fancyhdr}
\usepackage{authblk}
\usepackage{cleveref}
\usepackage{subcaption}
\usepackage{wrapfig}
\usepackage[toc, page]{appendix}

\newtheorem{lemma}{Lemma}
\newtheorem{coro}{Corollary}
\newtheorem{assum}{Assumption}

\newtheorem{theorem}{Theorem}

\theoremstyle{remark}

\newcommand{\R}{\mathbb{R}}

\newcommand{\N}{\mathbb{N}}

\newcommand{\e}{\begin{equation}}
\newcommand{\ee}{\end{equation}}
\newcommand{\en}{\begin{equation*}}
\newcommand{\een}{\end{equation*}}
\newcommand{\eqn}{\begin{eqnarray}}
\newcommand{\eeqn}{\end{eqnarray}}
\newcommand{\bmat}{\begin{bmatrix}}
\newcommand{\emat}{\end{bmatrix}}

\DeclareMathAlphabet\mathbfcal{OMS}{cmsy}{b}{n}

\newcommand{\E}{\operatorname{\mathbb{E}}}

\newcommand{\mb}{\bm}
\newcommand{\mc}{\mathcal}

\newcommand{\bb}{\mathbb}

\newcommand{\vct}[1]{\boldsymbol{#1}}
\newcommand{\mtx}[1]{\boldsymbol{#1}}

\newcommand{\diag}{\operatorname{diag}}

\newcommand{\wh}{\widehat}

\newcommand{\wt}{\widetilde}

\newcommand{\calA}{\mathcal{A}}
\newcommand{\calB}{\mathcal{B}}
\newcommand{\calC}{\mathcal{C}}
\newcommand{\calD}{\mathcal{D}}
\newcommand{\calE}{\mathcal{E}}

\newcommand{\calN}{\mathcal{N}}
\newcommand{\calO}{\mathcal{O}}

\newcommand{\calS}{\mathcal{S}}

\hypersetup{
    colorlinks=true,%
    citecolor=blue,%
    filecolor=blue,%
    linkcolor=blue,%
    urlcolor=blue
}

 \newcommand{ \Brac }[1]{\left\lbrace #1 \right\rbrace}

\newcommand{\vu}{\vct{u}}
\newcommand{\vv}{\vct{v}}

\newcommand{\vx}{\vct{x}}

\newcommand{\mO}{\mtx{O}}

\newcommand{\mU}{\mtx{U}}
\newcommand{\mV}{\mtx{V}}
\newcommand{\mW}{\mtx{W}}
\newcommand{\mX}{\mtx{X}}

\newcommand{\mOmega}{\mtx{\Omega}}
\newcommand{\mPhi}{\mtx{\Phi}}

\newcommand{\mTheta}{\mtx{\Theta}}

\graphicspath{{./figs/}}

\newlength{\imgwidth}
\newboolean{twoColVersion}
\setboolean{twoColVersion}{false}
\newcommand{\twoCol}[2]{\ifthenelse{\boolean{twoColVersion}} {#1} {#2} }

\long\def\comment#1{}

\def\E{\mathop{\rm E\,\!}\nolimits}

\newcommand{\bu}{\boldsymbol{u}}
\newcommand{\bv}{\boldsymbol{v}}

\newcommand{\bx}{\boldsymbol{x}}

\newcommand{\bz}{\boldsymbol{z}}
\newcommand{\bA}{\boldsymbol{A}}

\newcommand{\bI}{\boldsymbol{I}}

\newcommand{\bU}{\boldsymbol{U}}
\newcommand{\bV}{\boldsymbol{V}}
\newcommand{\bW}{\boldsymbol{W}}

\newcommand{\bY}{\boldsymbol{Y}}

\newcommand{\paren}[1]{\left( #1 \right)}
\long\def\red#1{\bgroup\color{red}#1\egroup}

\definecolor{mich-blue}{HTML}{0027CC}
\definecolor{mich-blue-high}{HTML}{0027CC}
\definecolor{red-high}{HTML}{CA2020}
\definecolor{green-high}{HTML}{20A520}
\definecolor{mich-maize}{HTML}{FFCB05}
\definecolor{law-stone}{HTML}{655A52}
\definecolor{burton-beige}{HTML}{9B9A9D}
\definecolor{arch-ivy}{HTML}{7E732F}

 \colorlet{color1}{gray!15}

%% file: main_sections/intro.tex
In recent years, deep learning has demonstrated remarkable success across a wide range of applications in engineering and science \cite{lecun2015deep}. 
Numerous studies have shown that the effectiveness of deep learning is partially due to the implicit bias of its learning dynamics, which favors some particular solutions that generalize exceptionally well without overfitting in the over-parameterized setting \cite{neyshabur2017implicit,nakkiran2021deep,belkin2019reconciling,huh2023simplicitybias}. To gain insight into the implicit bias of deep networks, a line of recent work has shown that gradient descent (GD) tends to learn simple functions \cite{shah2020pitfalls,valle-perez2018deep,gunasekar2018implicit,ji2018gradient,kunin2022asymmetric}. For instance, some studies have shown that gradient descent is biased towards max-margin solutions in linear networks trained for binary classification via separable data \cite{soudry2018implicit,kunin2022asymmetric}.  
In addition to the simplicity bias, another line of work showed that deep networks trained by GD exhibit a bias towards low-rank solutions \cite{huh2023simplicitybias,gunasekar2017implicit}. The works \cite{gidel2019implicit,arora2019implicit} demonstrated that adding depth to a matrix factorization enhances an implicit tendency towards low-rank solutions, leading to more accurate recovery. 


Despite the abundant empirical evidence in practical nonlinear networks, most theoretical results are developed based on over-parameterized linear models \cite{gunasekar2018implicit,soudry2018implicit,huh2023simplicitybias,arora2019implicit}. 
Notably, deep linear networks (DLNs), defined by multiple hidden layers with identity activations between layers, have been widely used as prototypes of practical deep networks for studying their nonlinear learning dynamics \cite{saxe2013exact,saxe2019mathematical,arora2018optimization,gidel2019implicit}.  Moreover, despite its simplicity, some properties of DLNs resemble those of their nonlinear counterparts. For example, for both linear and nonlinear networks, the work \cite{huh2023simplicitybias} empirically showed that the low-rank bias exists at both initialization and after training and is resilient to the choice of hyper-parameters and learning methods. The work \cite{saxe2019mathematical} showed that a DLN exhibits a striking, hierarchical progressive differentiation of structures in its internal hidden representations, which is similar to its nonlinear counterparts. Additionally, \cite{bell2019blind,de2020deep} demonstrated the practical applications of DLNs.
\begin{figure}
\begin{minipage}{.49\textwidth}
\begin{center}
\vspace{-1.5em}
\includegraphics[width=1\linewidth]{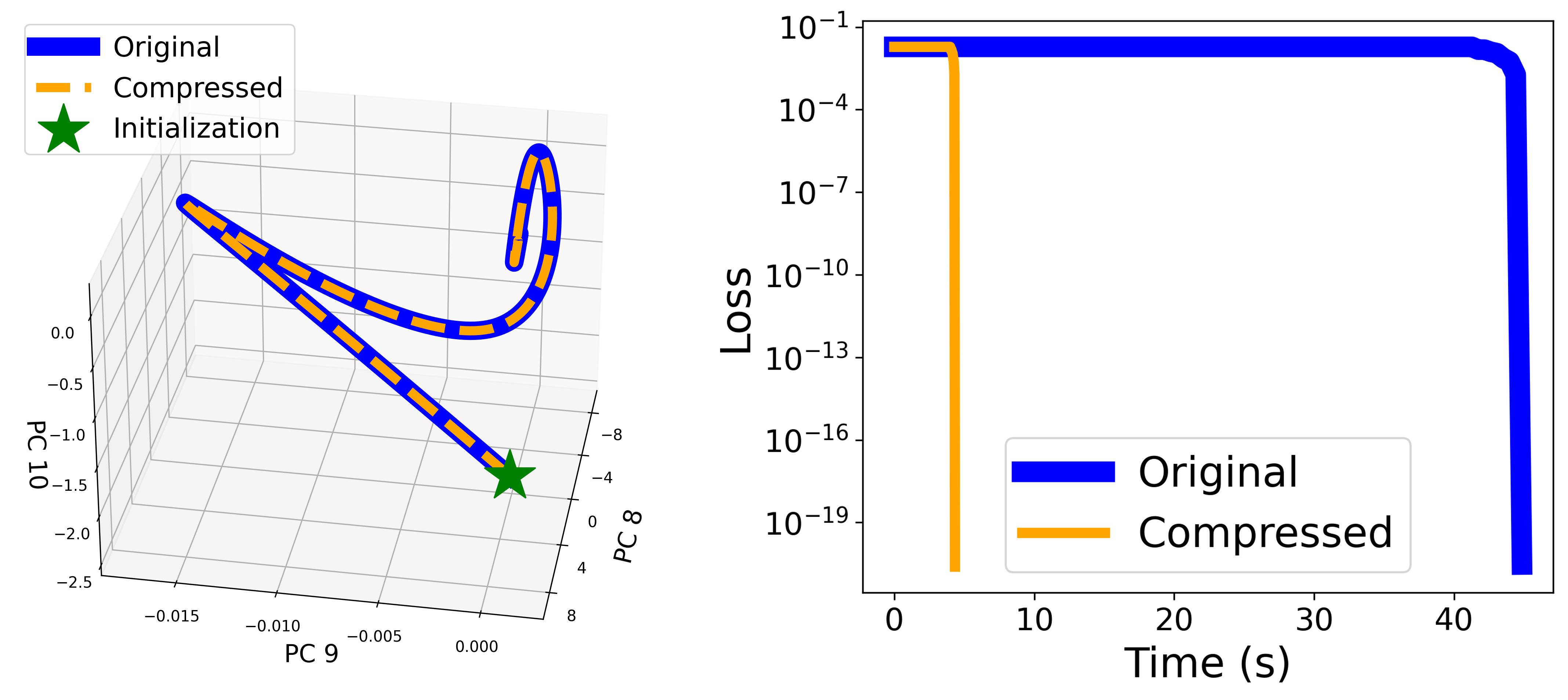}
\end{center}
\vspace{-1.5em}
\caption{\textbf{Efficient training of deep linear networks.} \textit{Left}: Principal components of end-to-end GD trajectories for overparameterized DLN and its equivalent compressed network. \textit{Right}: Training loss vs. wall-time comparison.}
\label{fig:traj_loss}
\end{minipage}
\hfill 
\begin{minipage}{.49\textwidth}
\vspace{-2.7em}
\begin{center}
    \includegraphics[width=1\linewidth]{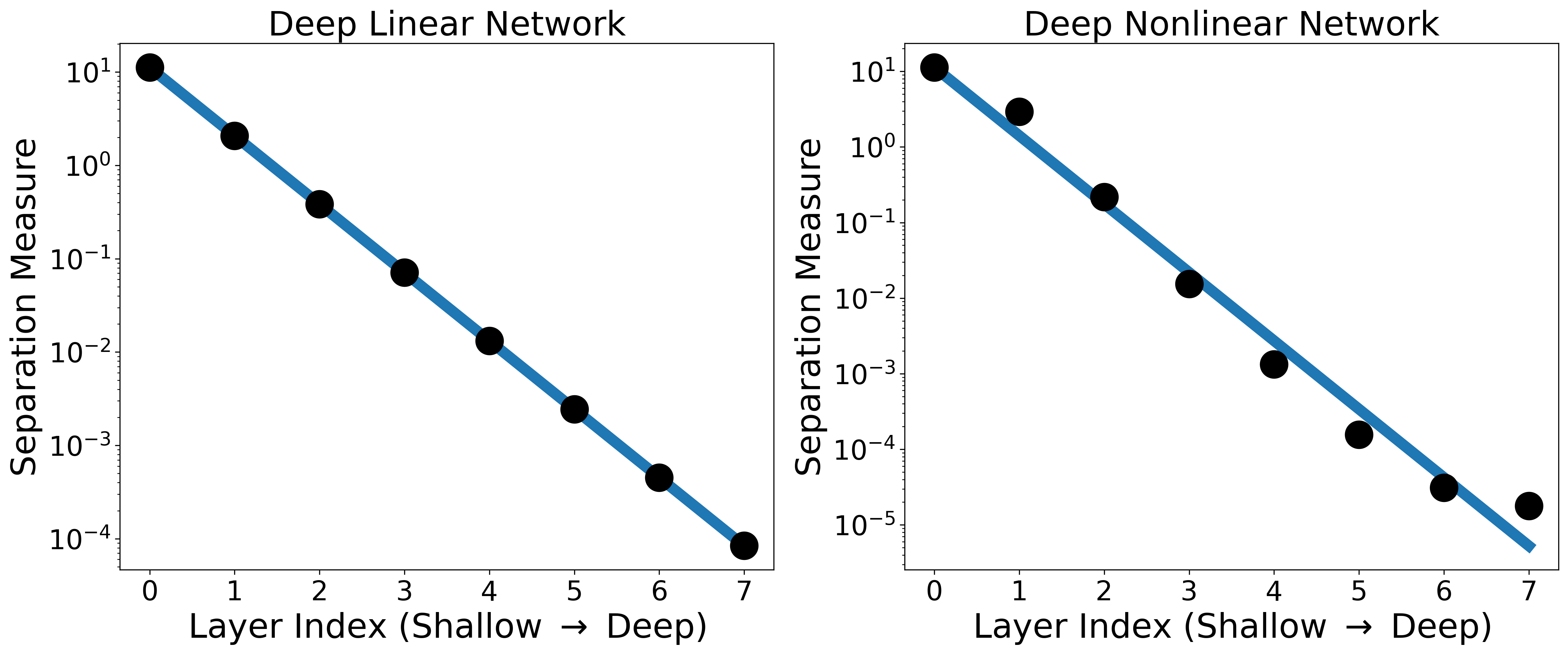}
\end{center}
\vspace{-0.9em}
\caption{\textbf{Progressive collapse with linear decay on deep \emph{linear} and \emph{nonlinear} networks.} The $x$-axis denotes the layer index and the $y$-axis denotes the separation measure in \eqref{eq:D measure}.}
\label{fig:progressive NC}
\end{minipage}
\vspace{-0.15in}
\end{figure}


\vspace{-0.1in}
\paragraph{Contributions.} In this work, we investigate a particular implicit bias of GD during the training of DLNs. When the cross-correlation matrix of the training data exhibits a low-dimensional structure, which manifests as either a low-rank or wide matrix in our context,
we show that the dynamics of GD tend towards parsimonious solutions. 
By examining the evolution of GD started from orthogonal initialization, we unveil a surprising, yet common phenomenon:
\begin{center}
    \emph{The learning process takes place only within a \textbf{minimal invariant subspace} of each weight matrix, while the remaining singular subspaces stay \textbf{unaffected} throughout training.}
\end{center}
Here, ``unaffected'' means that the remaining singular subspaces are not influenced or altered by the GD updates throughout the entire training process.
Notably, this phenomenon, which we term the ``law of parsimony'', persists despite the fact that GD updates all parameters of the weight matrices. When we use initialization of small scale, our work sheds new light on the implicit low-rank bias of the GD trajectory, explicitly explaining why deeper networks favor low-rank solutions throughout the entire training process \cite{huh2023simplicitybias}. Moreover, we demonstrate that such a phenomenon could have far-reaching implications for both understanding and improving the training efficiency of DLNs, which we highlight below.  
\begin{itemize}[leftmargin=*]
    \item \textbf{Dramatically more efficient training by constructing ``equivalent'' smaller networks.} As illustrated in \Cref{fig:traj_loss}, since learning only happens within a small \emph{invariant} subspace of the weights, we can construct and train significantly smaller DLNs that share the same learning dynamics as their wider counterparts. As such, we can significantly reduce the computational complexity of training deep networks without sacrificing the benefits of the associated wider networks, achieving the best of both worlds. We experimentally demonstrate such benefits on applications like deep matrix completion \cite{arora2019implicit}, showing that we can obtain the same sampling complexity as the original network while significantly improving the optimization efficiency. We believe such a finding could also have broad applications in training and fine-tuning practical deep networks \cite{hu2022lora}.
    \item \textbf{New theoretical insights into the progressive data separation in deep representation learning.} We show that the law of parsimony plays an important role in deciphering representation learning in the context of multi-class classification problems. Recent works \cite{papyan2020traces,he2022law,li2022principled,rangamani2023feature,xie2022hidden} have experimentally shown that the features across layers of a deep network exhibit a certain law of data separation that manifests in modern deep architectures during the terminal phases of training. Specifically, each layer of a trained network roughly improves a certain measure of data separation by an \textit{equal} multiplicative factor \cite{he2022law}, which is shown in \Cref{fig:progressive NC}. In this work, we theoretically investigate this phenomenon based on DLNs with orthogonal data. Thanks to the ``law of parsimony'' in the GD trajectory, we can precisely characterize the \emph{linear} progressive separation and concentration of representations from shallow to deep layers, potentially shedding new light on understanding the phenomenon in deep nonlinear networks.
\end{itemize}

\paragraph{Notations and Organization.}  
Let $\R^n$ be the $n$-dimensional Euclidean space and $\|\cdot\|$ be the Euclidean norm. 
Given any $n \in \mathbb{N}$, We use $\mb I_n \in \bb R^{n \times n} $ to denote an identity matrix of size $n$, and $\mb 1_n \in \bb R^n $ denote an all one vector of length $n$. Given any $L \in \mathbb{N}$, we use $[L]$ to denote the index set $\{1,\cdots,L\}$.  Let $\mathcal{O}^{m\times n} = \{\bm{X} \in \R^{m\times n}: \bm{X}^\top \bm{X} = \bm{I}_n\}$ denote the set of all $m\times n$ orthogonal matrices and $\calO^n$ the set of all $n \times n$ orthogonal matrices. 

The rest of the paper is organized as follows. In \Cref{sec:problem}, we introduce the basic problem setup. We present our main theoretical result in \Cref{sec:results}. The applications in deep matrix factorization and progressive collapse are demonstrated in \Cref{sec:low-rank} and \Cref{sec:separation}, respectively. 


%% file: main_sections/problem.tex
\vspace{-0.05in}
\paragraph{Basic Setup of Deep Linear Networks.}
Suppose that we have $N$ training samples $\{(\bm{x}_i,\bm{y}_i)\}_{i = 1}^N \subset \R^{d_x} \times \R^{d_y}$. Let $\mb X = \left[\mb x_1\ \mb x_2\ \dots\ \mb x_N
   \right] \in \bb R^{d_x \times N} $ and $\mb Y = \left[
\mb y_1\ \mb y_2\ \dots\ \mb y_N\right] \in \bb R^{d_y \times N}$ and define $\mb Y \mb X^\top$ to be the cross-correlation matrix. The goal of training a deep network is to learn a parameterized, hierarchical function $f: \bb R^{d_x} \mapsto \bb R^{d_y}$ that maps an input $\mb x_i \in \bb R^{d_x}$ to its corresponding label $\mb y_i \in \bb R^{d_y}$ for all $1\leq i \leq N$. In this work, we consider an $L$-layer ($L \geq 2$) linear network $ f_{\mb \Theta}(\cdot): \R^{d_x} \rightarrow \R^{d_y}$, parameterized by $\bm \Theta = \{\bW_l\}_{l=1}^L$ with input $\bm x \in \R^{d_x}$, i.e.,
\begin{equation}\label{eq:DLN}
    f_{\mb \Theta} (\vx) \;:=\; \mW_L \cdots \mW_1 \vx \;=\; \mb W_{L:1}\mb x ,
\end{equation}
where $\bm W_1 \in \R^{d_1 \times d_x}$, $\bm W_l \in \R^{d_{l} \times d_{l-1}}\;(l=2,\dots,L-1)$, and $\bm W_L \in \R^{d_y \times d_{L-1}}$ are weight matrices. For convenience, throughout the paper we adopt the abbreviations $\mW_{j:i} = \mW_j \cdots \mW_i$ and $\mW_{j:i}^\top = \mW_i^\top \cdots \mW_j^\top$ for $j \geq i$, where both are identity if $j < i$. 

To learn the network parameters $\bm \Theta $, we consider minimizing the $\ell_2$ loss on the training data $\{(\bm{x}_i,\bm{y}_i)\}_{i = 1}^N \subseteq \R^{d_x} \times \R^{d_y}$ as
\begin{align}\label{eq:obj}
    \min_{\bm \Theta}\ \ell (\bm \Theta) = \frac{1}{2}\sum_{i=1}^N \left\|f_{\bm \Theta}(\bm{x}_i) - \bm{y}_i \right\|_F^2  =  \frac{1}{2} \left\|{\bm W}_{L:1}\bm{X} - \bm{Y}\right\|_F^2.
\end{align}

\vspace{-0.1in} 
\paragraph{Training DLNs via GD.}

As the network is often over-parameterized, a common approach to enforce implicit regularization in solving the problem is to utilize GD starting from small initialization \cite{arora2019implicit,gidel2019implicit}. Here, we use $\mb W_l(t)$ to denote the weight of the $l$-th layer at the $t$-th iteration.
\begin{itemize}[leftmargin=*]
\vspace{-0.05in} 
    \item \textbf{Orthogonal initialization.} We initialize the weight matrices $\bW_l(0)$ for all $l \in [L]$ using $\varepsilon$-scaled orthogonal matrices for some $\varepsilon > 0$, i.e., 
    \begin{align}\label{eq:init}
        \mb W_l^\top(0)\mb W_l(0) = \varepsilon^2\mb I \quad \mbox{or} \quad \mb W_l(0)\mb W_l^\top(0) = \varepsilon^2\mb I, \quad \forall l \in [L],
    \end{align}
    which depends on the size of $\mW_l$. It is worth noting that orthogonal weight initialization is a commonly employed technique in neural network training due to its ability to speed up the convergence of GD \cite{chen2018dynamical,pennington2018emergence,saxe2013exact,xiao2018dynamical,hu2020provable}. 
    \item \textbf{Learning dynamics of GD.} We update all weights via GD for $t = 1,2,\dots$ as 
    \begin{align*}
        \bm{W}_l(t) = (1-\eta\lambda) \bm{W}_l(t-1) - \eta \nabla_{\mW_l} \ell(\mTheta(t-1)),\;\; \forall\; l \in [L],
    \end{align*}
    where $\lambda \ge 0$ is an optional weight decay parameter and $\eta>0$ is the learning rate. Substituting the explicit form of the gradient $\nabla_{\mW_l} \ell(\mTheta)$ of \eqref{eq:obj} into the above equation, we obtain the analytical form of GD as
    \begin{align}\label{eq:gd}
         \bm{W}_l(t) = (1-\eta\lambda) \bm{W}_l(t-1) -  \eta \bW_{L:l+1}^\top (t-1){\bm \Gamma}(t-1)\bW_{l-1:1}^\top (t-1),
    \end{align}
    where we denote ${\bm \Gamma}(t) = \paren{\bW_{L:1}(t)\bm{X} - \bm{Y}}\bm{X}^\top$ for simplicity. 
\end{itemize}

%% file: main_sections/results.tex
\vspace{-0.1in}
Before stating our main result, we make the following assumption for ease of exposition. %

\begin{assum}\label{AS:1}
The weight matrices are square except the last layer, i.e., $d_1 = d_2 = \cdots = d_{L-1} = d$ for some $d \in \N_+$. In particular, we have $d_x = d$. Also, the input data is {\em whitened} in the sense that $\bm{X}\bm{X}^\top = \bm{I}_{d_x}$.\footnote {For any full rank $\mX \in \R^{d_x \times N}$ with $N \geq d_x$, whitened data can always be obtained with a data pre-processing step such as preconditioning.}
\end{assum}

\noindent Note that the above assumptions can be relaxed to give very similar results to the ones we prove here.
For example, Assumption \ref{AS:1} can be potentially relaxed to $d_{L-1} \geq d_{L-2} \geq \cdots \geq d_1$. Moreover, empirical evidence suggests that our main results should approximately hold for any well-conditioned $\mX$, not necessarily restricted to whitened $\mb X$ -- we leave this study for future work. Based on the assumption above, we show that all iterates $\mb W_l(t)$ along the GD trajectory exhibit parsimonious structures when the 
cross-correlation matrix
possesses low-dimensional structures. 

\begin{figure}[t]
    \centering
    \vspace{-3em}
    \includegraphics[width=1\linewidth]{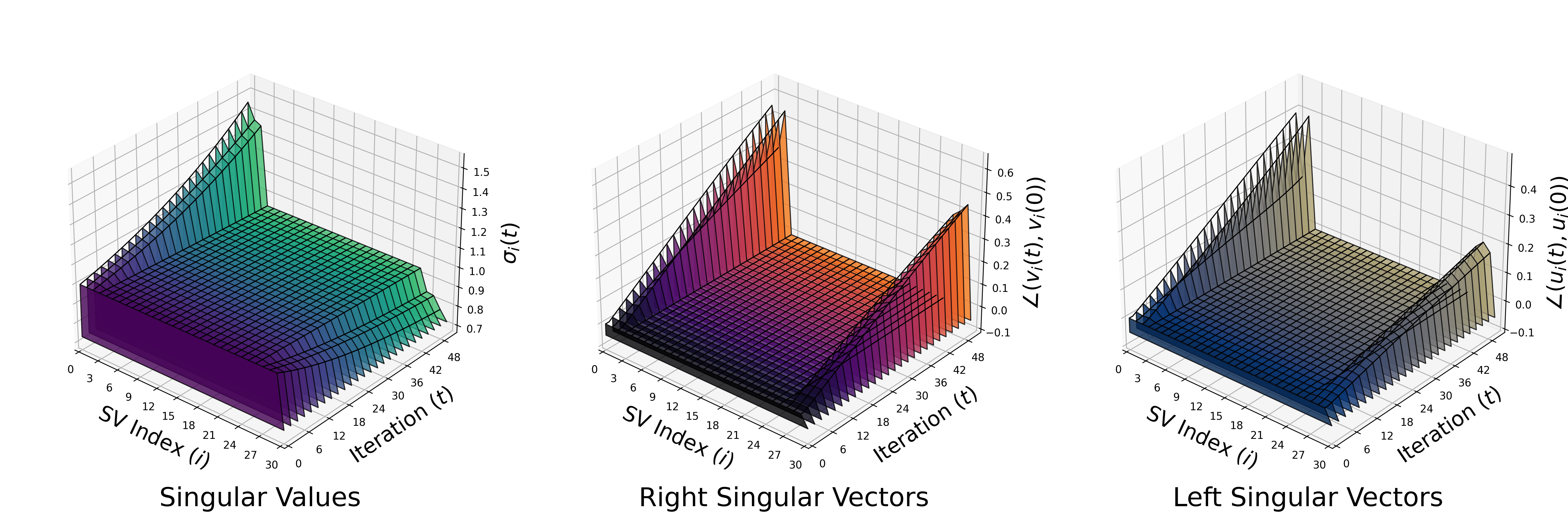}
    \caption{\textbf{Evolution of SVD of weight matrices.} We visualize the SVD of the first layer weight matrix of an $L=3$ layer deep linear network with $d_x = d_y = 30$ and $r=3$ (\textbf{Case 1}) throughout GD without weight decay. \textit{Left}: Magnitude  of the $i$-th singular value $\sigma_i(t)$ at iteration $t$. \textit{Middle}: Angle $\angle(\vv_i(t), \vv_i(0))$ between the $i$-th right singular vector at iteration $t$ and initialization. \textit{Right}: Angle $\angle(\vu_i(t), \vu_i(0))$ between the $i$-th left singular vector at iteration $t$ and initialization.
    }
    \label{fig:traj_matrices}
    \vspace{-0.15in} 
\end{figure}


\begin{theorem}\label{thm:main}
Suppose that an $L$-layer $f_{\mb \Theta}(\cdot)$ and the training data $(\mb X,\mb Y)$ satisfy Assumption \ref{AS:1}. We run GD \eqref{eq:gd} with weight decay parameter $\lambda$ and learning rate $\eta$ to train  $f_{\mb \Theta}(\cdot)$ starting from $\varepsilon$-scaled orthogonal initialization \eqref{eq:init}. Then, the iterates $\Brac{\mb W_l(t)}_{l=1}^L$ for all $t \ge 0$ possess parsimonious structures in the following sense:
\begin{itemize}
    \item[\textbf{Case 1.}] 
    Suppose the cross-correlation $\mb Y \mb X^\top \in \R^{d_y \times d_x}$ is of rank $r \in \N_+$ with $d_y = d_x$, and
    $m:= d_x - 2r > 0$. Then there exist orthogonal matrices $\{\mb U_{l}\}_{l=1}^{L} \subseteq \calO^d$ and $\{\mb V_l\}_{l=1}^L \subseteq \calO^d $ satisfying $\mb V_{l+1} = \mb U_{l}$ for all $l \in [L-1]$, such that $\mb W_l(t)$ admits the following decomposition
    \begin{align}\label{eq:weight_structures}
        \mb W_l(t) = \mb U_l \begin{bmatrix}
            \widetilde{\mb W}_l(t) & \bm{0} \\
            \bm{0} & \rho(t) \bm{I}_m
        \end{bmatrix}  \mb V_l^\top
    \end{align}
    for all $l \in [L]$ and $t \ge 0$, where $\wt{\mb W}_l(t)\in \bb R^{2r \times 2r}$ for all $l \in [L]$ with $\wt{\mW}_l(0) = \varepsilon \mb I_{2r}$, and 
    \begin{align}\label{eq:rho 1}
        \rho(t) = \rho(t-1)\left(1 - \eta\lambda -\eta \cdot \rho(t-1)^{2(L-1)}\right)
    \end{align}
    for all $t\geq 1$ with $\rho(0) = \varepsilon$.
    \item[\textbf{Case 2.}] Suppose the cross-correlation $\mb Y \mb X^\top \in \R^{d_y \times d_x}$
    with $d_y = r$ satisfies $m := d_x - 2d_y > 0$. Then, $\mb W_l(t)$ admits the same decomposition as in \eqref{eq:weight_structures} for all $l \in [L-1]$ and $t \geq 0$ except 
    \begin{align}\label{eq:rho 2}
        \rho(t)= \varepsilon \paren{1 - \eta \lambda }^t,\ \forall t \ge 0.
    \end{align}
\end{itemize}
\end{theorem}

\noindent We defer comprehensive comparison with prior arts and detailed proofs to \Cref{sec:conclusion} and \Cref{app:proof3}, respectively. To help the reader gain more insights from our results, we make several remarks in the following. 
\vspace{-0.1in}  
\paragraph{Dynamics of singular values and vectors of weight matrices.} It is worth noting that the decomposition \eqref{eq:weight_structures} is closely related to the singular value decomposition (SVD) of $\mb W_l(t)$. Specifically, let $\mb U_l = [\mb U_{l,1}\ \mb U_{l,2}]$, $\mb V_l = [\mb V_{l,1}\ \mb V_{l,2}]$, where $\mb U_{l,1},\mb V_{l,1} \in \mc O^{d\times 2r}$, $\mb U_{l,2},\mb V_{l,2} \in \mc O^{d\times (d-2r)}$. Let $\widetilde{\mb W}_l(t) = \widetilde{\bm U}_l(t)\widetilde{\bm \Sigma}_l(t)\widetilde{\bm V}_l^\top(t)$ be an SVD of $\widetilde{\mb W}_l(t)$, where $\widetilde{\bm U}_l(t),\widetilde{\bm V}_l(t) \in \calO^{2r}$ and $\widetilde{\bm \Sigma}_l(t) \in \R^{2r\times 2r}$ is a diagonal matrix. Then, we can rewrite \eqref{eq:weight_structures} into
\begin{align}\label{eq:SVD Wlt}
    \mb W_l(t) = \begin{bmatrix}
      \mb U_{l,1}\widetilde{\bm U}_l(t)  &  \mb U_{l,2}
    \end{bmatrix} \begin{bmatrix}
            \widetilde{\bm \Sigma}_l(t) & \bm{0} \\
            \bm{0} & \rho(t) \bm{I}_m
        \end{bmatrix}  
    \begin{bmatrix}
      \mb V_{l,1} \widetilde{\bm V}_{l}(t)  &  \mb V_{l,2}
    \end{bmatrix}^\top,
\end{align}
which is essentially an SVD of $\bm{W}_l(t)$ (besides the ordering of singular values). According to this, we can verify that the (repeated) singular value $\rho(t)$ undergoes minimal changes across iterations when $\varepsilon$ is small according to \eqref{eq:rho 1} and \eqref{eq:rho 2} -- this is illustrated in \Cref{fig:traj_matrices} (left). 
\vspace{-0.1in}  
\paragraph{Low-rank implicit bias.} We emphasize that our result sheds new light on the implicit low-rank bias of GD. Specifically, it follows from \eqref{eq:rho 1} and \eqref{eq:rho 2} that $\lim_{\varepsilon \rightarrow 0} \rho(t) = 0$ for all $t \ge 0$. This, together with \eqref{eq:SVD Wlt}, implies that the dynamics of GD are inherently biased towards finding low-rank solutions with a rank of at most $2r$. In contrast to existing work that demonstrates the tendency of GD to find low nuclear-norm solutions \cite{gunasekar2017implicit,arora2019implicit}, we directly show that GD tends to find low-rank solutions.  
\vspace{-0.1in}   
\paragraph{Invariance of subspaces.} According to \eqref{eq:SVD Wlt}, it is evident that the subspace of dimension $d-2r$ formed by left (resp. right) singular vectors in $\mb U_{l,2}$ (resp. $\mb V_{l,2}$) corresponding to the singular values in $\rho(t)$ remains unchanged during iterations; see \Cref{fig:traj_matrices}. This indicates that the learning process occurs only within an invariant subspace of dimension $2r$. This result allows us to gain insights into and improve the training efficiency of DLNs for deep matrix completion (\Cref{sec:low-rank}), as well as elucidate the linear progressive separation of representations of DLNs (\Cref{sec:separation}). 
\vspace{-0.1in}   
\paragraph{Comparison to prior arts.} In our analysis, we specifically investigate the impact of weight decay on the implicit bias of GD. Unlike previous work on implicit bias \cite{min2022convergence,gissin2019implicit,arora2019implicit,vardi2021implicit}, which did not explicitly consider weight decay, we carefully examine the effect of this regularization technique. In particular, when weight decay regularizer $\lambda > 0$ is applied, we observe that the singular value $\rho(t)$ tends to zero asymptotically as $t$ goes to infinity. This finding indicates that gradient descent with weight decay is biased towards finding low-rank solutions. Moreover, we believe our result can be generalized to other optimization methods beyond GD, such as Adam \cite{kingma2014adam}, AdaGrad \cite{duchi2011adaptive}, and RMSprop \cite{tieleman2012rmsprop}.

%% file: main_sections/application1.tex
\vspace{-0.05in} 
First, we demonstrate how the parsimonious structures of GD in DLNs can be applied to dramatically improve the optimization efficiency of solving deep matrix completion \cite{arora2019implicit}. 

\vspace{-0.1in} 
\paragraph{Problem Setup.} We consider the low-rank matrix completion problem \cite{candes2012exact,candes2010power,davenport2016overview} with ground-truth $\mPhi \in \R^{d\times d}$ with $r := \mbox{rank}(\mPhi) \ll d$. Our goal is to recover $\mb \Phi$ from as few number of observations as possible, where the observed entries are encoded by an index matrix $\mOmega \in \{0, 1\}^{d \times d}$. To solve the problem, we consider the recent deep matrix factorization approach \cite{arora2019implicit}, by optimizing a variant of Problem \eqref{eq:obj} with $\mb Y = \mb \Phi$ and identity input $\mb X = \mb I_d$, where the objective\footnote{The recovery error is defined by flipping the entries in $\mOmega$ above, i.e., the error in the unobserved entries. } is defined as
\begin{equation}\label{eq:obj_mc}
    \ell_{\mathrm{mc}}(\mTheta) := \frac{1}{2}\|\mOmega \odot (\mW_{L:1} - \mPhi)\|_F^2.
\end{equation}
When the complete observation $\mOmega = \bm 1_d \bm 1_d^\top$ is available, the above problem simplifies to deep matrix factorization, as depicted in Problem \eqref{eq:obj}. Moreover, if the network depth is $L=2$, Problem \eqref{eq:obj_mc} reduces to a Burer-Monteiro factorization \cite{burer2003nonlinear}. Despite its nonconvexity, significant advances have been made in understanding its global optimality and GD convergence under various settings in the past few years \cite{jain2013low,zheng2016convergence,sun2016guaranteed,ge2016matrix,bhojanapalli2016global,ge2017no,gunasekar2017implicit,li2019non,chi2019nonconvex,li2018algorithmic,soltanolkotabi2023implicit}.


\vspace{-0.1in} 
\paragraph{Benefits of Deep Networks.} In practice, the true rank $r$ is often not known exactly. Instead, we may have a rough estimate of its upper bound $\wh{r}$, i.e., $r \leq \wh{r} \ll d$. When we consider Problem \eqref{eq:obj_mc} in the over-parameterized regime by overestimating the rank as $\wh{r} \geq r$, as recent work has shown that GD starting from small initialization has an implicit bias towards the true rank solution \cite{gunasekar2017implicit}, we do not need an explicit regularization and set the weight decay $\lambda =0$ when we run GD 
for optimizing Problem \eqref{eq:obj_mc}. Furthermore, in the over-parameterized regime, more recent work \cite{arora2019implicit} demonstrated that using a deeper network (i.e., $L\geq 3$) in solving Problem \eqref{eq:obj_mc} enjoys several substantial benefits over the shallow counterpart $L=2$. 
\begin{itemize}[leftmargin=*]
\vspace{-0.05in} 
    \item \textbf{Benefits of depth.} As shown in \Cref{fig:depth_2_v_3} (left),  when we increase over-parameterization, training deeper networks ($L=3$) with GD is less prone to overfitting. Additionally, the work \cite{arora2019implicit} has shown that deeper networks improve sample complexity over their shallow counterparts. It has also been shown that GD for deeper networks has a stronger implicit bias towards low-rank solutions \cite{huh2023simplicitybias}.
    \item \textbf{Benefits of width.} On the other hand, increasing the width of the network results in accelerated convergence of GD in terms of iterations.  As shown in \Cref{fig:depth_2_v_3} (right), increasing the network width for a 3-layer network reduces the number of GD iterations needed to converge.
\end{itemize}
\noindent Nonetheless, the advantages of deeper and wider networks are accompanied by computational trade-offs: Increasing both the depth and width of a network significantly increases the number of parameters that need to be optimized, thereby causing the per-iteration cost of minimizing Problem \eqref{eq:obj_mc} via GD to be much greater than that of shallower and narrower networks. In the following, we show that this computational challenge associated with training deeper and wider networks can be mitigated by invoking the law of parsimony in \Cref{thm:main}, where we can construct an approximately equivalent but considerably smaller network to speed up the optimization process.



\begin{figure}
\begin{minipage}{.49\textwidth}
\begin{center}
\includegraphics[width=\linewidth]{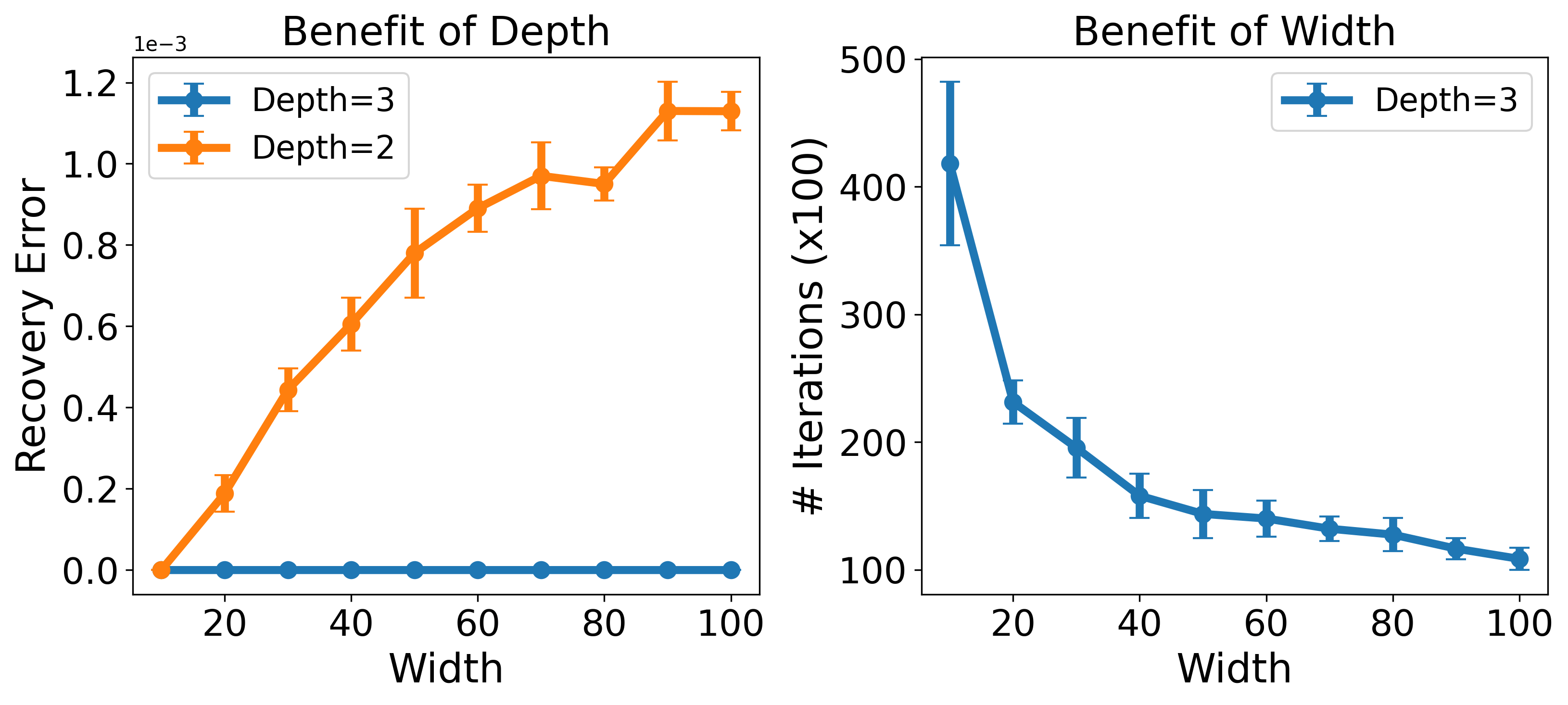}
  \caption{\textbf{Benefits of depth and width in overparameterized matrix completion.} \textit{Left}: Recovery error using $L=2$ vs. $L=3$. \textit{Right}: Number of GD iterations to converge.}
  \label{fig:depth_2_v_3}
  \end{center}
\end{minipage}
\hfill 
\begin{minipage}{.49\textwidth}
  \includegraphics[width=\linewidth]{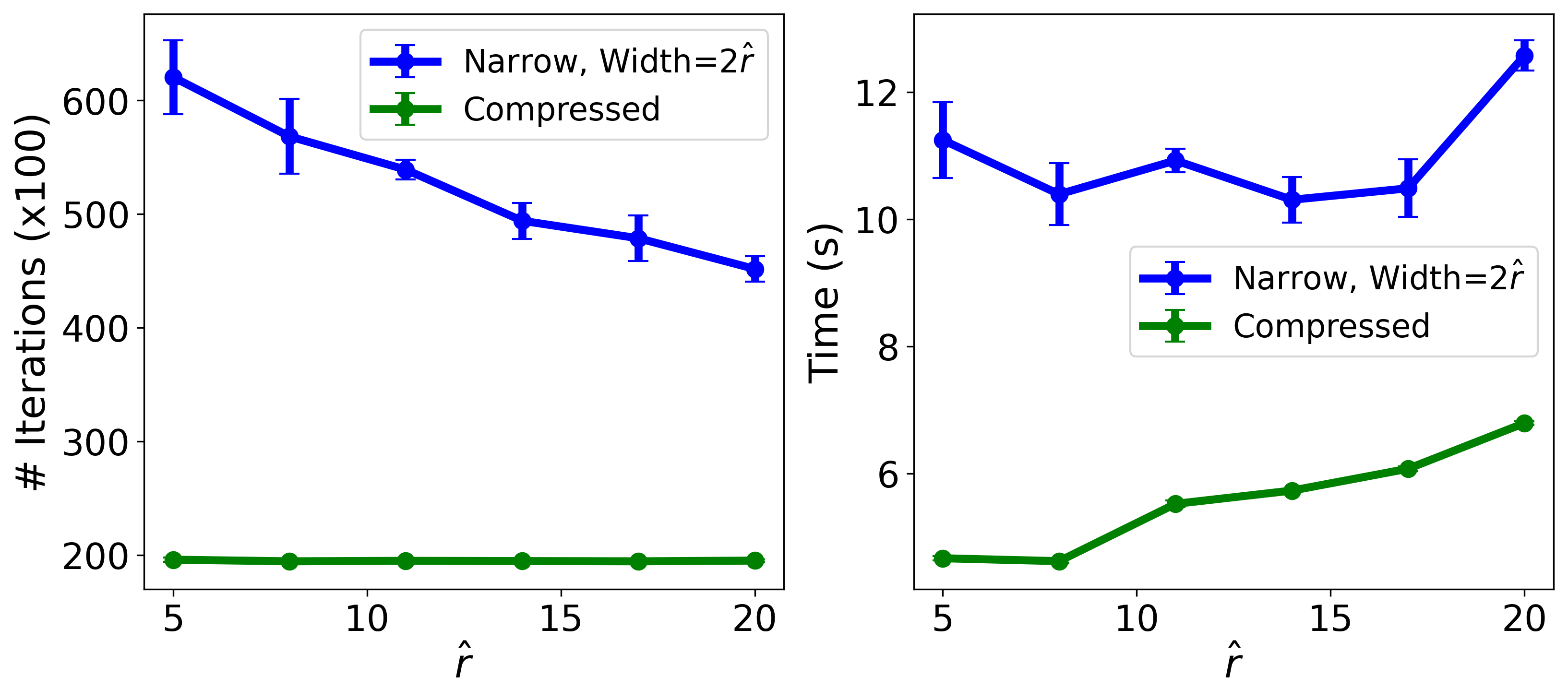}
  \caption{\textbf{Comparison of efficiency of compressed networks vs. narrow network} with different overestimated $\wh{r}$. \textit{Left}: Number of iterations to converge. \textit{Right}: Time to converge.}
  \label{fig:comp_iter_time}
\end{minipage}
\vspace{-0.15in}
\end{figure}
\vspace{-0.1in} 
\paragraph{Network Compression via Law of Parsimony.} To build up intuition, we first describe our approach in a simplified setting of deep matrix factorization, where we have full observation of $\mb \Phi$. Then we extend the idea to over-parameterized low-rank matrix completion. 


\noindent \textbf{\emph{$\blacktriangleright$ The Deep Matrix Factorization Setting:}} With $\mOmega = \bm 1_d \bm 1_d^\top$, \eqref{eq:obj_mc} now reduces to the vanilla problem \eqref{eq:obj} where we can apply \Cref{thm:main} with $m=d - 2\wh{r} > 0$.\footnote{We note that we can replace $r$ in \Cref{thm:main} with $\wh{r}$ and maintain the same dynamics provided that $d - 2\wh{r} > 0$. } Based on \eqref{eq:weight_structures} as well as the fact that $\mb V_{l+1} = \mb U_l$ for all $l \in [L-1] $ from \Cref{thm:main}, we can always write the end-to-end matrix $\mW_{L:1}(t)$ as
\begin{align*}
    \mW_{L:1}(t) =  \begin{bmatrix}
        \mb U_{L,1} & \mb U_{L,2}
    \end{bmatrix} \begin{bmatrix}  \wt{\mW}_{L:1}(t)  & \mb 0 \\
    \mb 0 & \rho^L(t) \mb I_m \end{bmatrix} \begin{bmatrix}
        \mb V_{1,1}^\top \\ \mb V_{1,2}^\top
    \end{bmatrix}
    = \mU_{L,1} \wt{\mW}_{L:1}(t) \mV_{1,1}^\top + \rho^{L}(t) \mU_{L, 2}\mV_{1, 2}^\top
\end{align*}
for all $t\geq 0$, where $\wt{\mW}_{L:1}(t) \in \bb R^{2\wh{r} \times 2\wh{r}}$ is the end-to-end matrix for all the compressed weights, $\mb U_{L,1},\mb V_{1,1} \in \mc O^{d\times 2\wh{r}}$, and $\mb U_{L,2},\mb V_{1,2} \in \mc O^{d\times m }$. Then, our claim for deep matrix factorization is that:
\begin{center}
    \emph{For optimizing \eqref{eq:obj_mc} with small initialization, running GD on the original weights $\Brac{\mb W_l}_{l=1}^L \subseteq \bb R^{d \times d} $\ is \textbf{almost equivalent} to running GD on the compressed weights $\{\wt{\mb W}_l\}_{l=1}^L \subseteq \bb R^{2\wh{r} \times 2\wh{r}} $.}
\end{center}
As a result, since the compressed weights consist of only $4L\wh{r}^2$ parameters as opposed to the $Ld^2$ parameters of the original weights, their optimization can be significantly more efficient when $\wh{r} \ll d$. To numerically verify our claim, we train the compressed network for target $\mPhi$ with $d=1000$, $\wh{r} = r = 5$ in \Cref{fig:traj_loss}. It can be observed that the end-to-end GD trajectory of optimizing the compressed network closely follows the trajectory of the original network, while converging to the optimal solution an entire order of magnitude earlier. While our claim may not possess complete rigor, let us briefly outline two fundamental components to elucidate why our claim above holds true:
\begin{itemize}[leftmargin=0.25in]
    \item \textbf{The effects of small initialization $\varepsilon$ and depth $L$.} From \eqref{eq:rho 1}, we know that $\rho(t)$ depends on $\varepsilon$, with $ \lim_{\varepsilon \rightarrow 0} \rho(t) = 0 $.
    As such, the term $\rho^{L}(t) \cdot \mU_{L, 2}\mV_{1, 2}^\top \approx \mb 0$ when we use small initialization $\varepsilon\approx 0$, and its size decreases with larger depth $L$. Therefore, we have
    \begin{align}\label{eq:f0}
     \mb W_{L:1}(t) \;=\; \mU_{L,1} \wt{\mW}_{L:1}(t) \mV_{1,1}^\top + \rho^{L}(t) \mU_{L, 2}\mV_{1, 2}^\top \;\approx\; \mU_{L, 1} \wt{\mW}_{L:1}(t) \mV_{1,1}^\top,\quad \forall t\geq 0.
    \end{align} 
    \item \textbf{Invariance of weight subspaces and GD dynamics throughout training.}   
    Moreover, we know from \Cref{thm:main} that all $\Brac{\mU_{l,1}}_{l=1}^L$ and $\Brac{\mV_{l,1}}_{l=1}^L$ remain unchanged throughout the GD dynamics, and gradients commute under orthogonal transformations across different layers. Therefore, running GD on the original weights $\mW_l(t)$ is essentially equivalent to running GD on the compressed weights $\wt{\mW}_l(t)$ for all $t\geq 0$ and $l \in [L]$. 
\end{itemize}

\begin{figure}[t]
  \centering
  \includegraphics[width=0.95\linewidth]{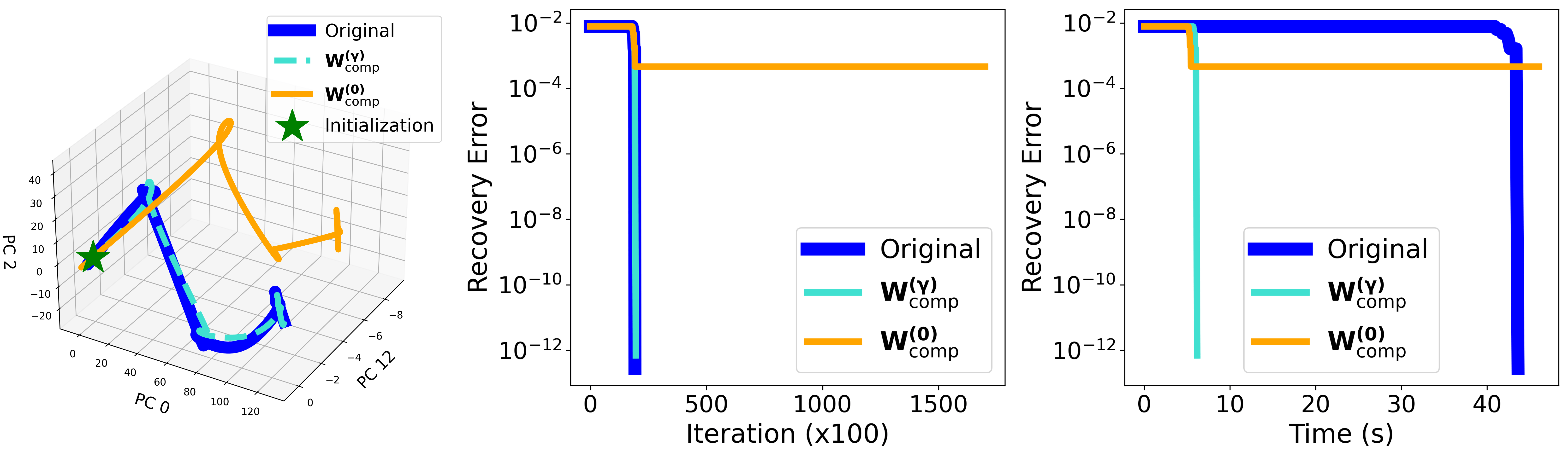}
  \caption{\textbf{Ablation of $\gamma$ in compressed network for deep matrix completion.}
  Comparison of original overparameterized 3-layer network and its approximate compressed networks $\mW_{\mathrm{comp}}^{(\gamma)}$ ($\gamma > 0$) and $\mW_{\mathrm{comp}}^{(0)}$ ($\gamma = 0$). \textit{Left}: Principal components of end-to-end trajectories of each network. \textit{Middle}: Recovery error vs. iteration comparison. \textit{Right}: Recovery error vs wall-time comparison.}
  \label{fig:equiv_traj_mc}
  \vspace{-0.15in}
\end{figure}

\noindent \textbf{\emph{$\blacktriangleright$ Extension to Deep Matrix Completion:}} 
We should mention that our result for deep matrix factorization cannot be directly applied to deep matrix completion -- this is due to the fact that the observation matrix $\mOmega \odot \mPhi$ for a generic $\mOmega$ does not necessarily have the low-rank structure of $\mb \Phi$, and hence the $\mV_{1, 1}$ and $\mU_{L, 1}$ factors in \eqref{eq:f0} do \emph{not} remain unchanged throughout the GD iterations for all $t\geq 0$. To deal with this issue, we propose to update both $\mV_{1, 1}(t)$ and $\mU_{L, 1}(t)$ factors via GD with learning rate $\gamma \eta$ in
\begin{equation}\label{eq:fgamma}
    \mb W_{\mathrm{comp}}^{(\gamma)}(t) := \mU_{L, 1}(t) \wt{\mW}_{L:1}(t) \mV_{1, 1}^\top(t).
\end{equation}
This is done simultaneously with the GD updates on the subnetwork $\wt{\mW}_{L:1}(t)$, which uses the original learning rate $\eta$. We call $\mb W_{\mathrm{comp}}^{(\gamma)}(t)$ the \emph{compressed network}, where $\gamma \in(0,1)$ denotes the discrepancy in the learning rate. More specifically, we make the following modifications to our method:
\begin{itemize}[leftmargin=0.25in]
    \item \textbf{Initialize $(\mV_{1, 1}(0),\mU_{L, 1}(0))$.} We initialize the factor $(\mV_{1, 1}(0),\mU_{L, 1}(0))$ using the ones described in \eqref{eq:f0}, but they are calculated based upon $\mOmega \odot \mPhi$ instead of $\mb \Phi$.
    \item \textbf{Update $(\mV_{1, 1}(t),\mU_{L, 1}(t))$ with small learning rates.} Although the subspaces $(\mV_{1, 1}(t),\mU_{L, 1}(t))$ are changing, they are changing much more slowly compared to the compressed weights $\wt{\mb W}_l(t)$. Therefore, we update $\mV_{1, 1}(t)$ and $\mU_{L, 1}(t)$ using a discrepant and smaller learning rate $\gamma \eta$ with ($\gamma\ll 1$) for all $t\geq 0$, compared to the learning rate $\eta$ used for the compressed weights $\wt{\mb W}_l(t)$.
\end{itemize}
As illustrated in \Cref{fig:equiv_traj_mc} (left), monitoring the end-to-end GD trajectory on the compressed network $\mb W_{\mathrm{comp}}^{(\gamma)}(t)$ in \eqref{eq:fgamma} reveals that maintaining $(\mU_{L, 1}(t),\mV_{1, 1}(t))$ unchanged (i.e., $\gamma=0$) results in the GD trajectory of $\mb W_{\mathrm{comp}}^{(0)}(t)$ deviating from that of the original network during the later stages of training, leaving the final test error fairly high. This is due to the accumulation of approximation error throughout the final training phase. On the other hand, updating $(\mV_{1, 1}(t),\mU_{L, 1}(t))$ using a small learning rate $\gamma \eta$ with $\gamma=0.01$ ensures that the GD trajectory of $\mb W_{\mathrm{comp}}^{(\gamma)}(t)$ closely mirrors that of the original network, resulting in a significantly reduced recovery error upon convergence of training.


Regarding computation, it is worth noting that the modified approach requires optimizing an additional $4\wh{r}d$ parameters compared to the deep matrix factorization scenario -- this is due to the additional updates on $(\mV_{1, 1}(t),\mU_{L, 1}(t))$. Nevertheless, compared to optimizing the original network, our modified approach remains considerably more efficient, optimizing only $O(d)$ parameters as opposed to the $O(d^2)$ parameters required for the original network. Our experimental result in \Cref{fig:equiv_traj_mc} (right) also supports this, demonstrating that the compressed network converges significantly faster in terms of time compared to the original network.



\vspace{-0.1in}
\paragraph{Compressed Networks vs. Narrow Networks.} 
Despite the intriguing observation that the effective rank of the changing subspace is capped at $2\wh{r}$, it still prompts the following question: \emph{Does this imply that optimizing a narrow network of the same width $2\wh{r}$ would perform just as efficiently as the compressed network with a true width of $d\gg \wh{r}$?}


Our experimental results suggest that the answer is \emph{no} in general -- we compare the training efficiency of a $2\wh{r}$-compressed network (within a wide network of width $d \gg \wh{r}$) versus a narrow network with width 2$\wh{r}$ under different over-parameterized estimates $\wh{r}$. As depicted in \Cref{fig:comp_iter_time} (left), the compressed network requires fewer iterations to reach convergence, and the number of iterations necessary is almost unaffected by $\wh{r}$. Consequently, \emph{training compressed networks is considerably more time-efficient than training narrow networks of the same size}, provided that $\wh{r}$ is not significantly larger than $r$. The distinction between the compressed and narrow networks underscores the benefits of wide networks, as previously demonstrated and discussed in \Cref{fig:depth_2_v_3} (right), where increasing the network width results in faster convergence. However, increasing the network width alone also increases the number of parameters. By employing our network compression methodology, we can achieve the best of both worlds.
\vspace{-0.1in} 
\paragraph{Experimental Setups.} Regarding our experimental setups, for \Cref{fig:depth_2_v_3} we consider matrix completion problem with $d=100$, $r=10$, and $30$\% of entries observed, and for \Cref{fig:comp_iter_time} and \Cref{fig:equiv_traj_mc}, we consider a matrix completion problem with $d=1000$, $r=5$, and $20$\% of entries observed. We optimize deep networks with depth $L=3$ via GD starting from a small orthogonal initialization of scale $\varepsilon = 10^{-3}$, until the objective \eqref{eq:obj_mc} achieves a value less than $10^{-10}$.

%% file: main_sections/application2.tex
For the multi-class classification problem, we employ our general result in \Cref{thm:main} to demystify the intriguing phenomenon illustrated in \Cref{fig:progressive NC}. We refer the reader to \Cref{app:proof4} for the proofs of the results presented in this section. 
\vspace{-0.1in}
\paragraph{Problem Setup for Multi-Class Classification.} 
We consider a $K$-class classification problem with training data samples $\{(\bm{x}_{k,i},\bm{y}_k)\}_{ i \in [n_k], k \in [K] }$, where $\bx_{k,i} \in \mathbb R^{d}$ is the $i$-th sample in the $k$-th class, $\bm{y}_k \in \mathbb R^K$ is an one-hot label vector\footnote{A one-hot label, e.g., $\bm{y}_k$, has only the $k$-th entry equal to $1$ with the remaining entries equal to $0$.}, and the number of samples $n_k$ in each class is balanced with $n_1 = \cdots = n_K = n$. We denote the total number of samples by $N = nK $. Based upon the above, we train a $L$-layer DLN to learn weights $\bm \Theta = \{\bW_l\}_{l=1}^L$ via minimizing the $\ell_2$ loss in \eqref{eq:obj}, where $\bW_l \in \R^{d_l \times d_{l-1}}$, $d_0=d$, and $d_L=K$. We write the feature $\bm{z}_{k,i}^l$ of an input sample $\bx_{k,i}$ in the $l$-th layer as
\begin{align}\label{eq:zl}
\bm{z}_{k,i}^l := \bW_l \dots \bW_1 \bx_{k,i} = \mb W_{l:1} \mb x_{k,i},\ \forall l \in [L],
\end{align}
and we denote $\bz_{k,i}^0 = \bx_{k,i}$. To characterize the network's capability to separate data separation across layers on the whole training dataset, we use a metric introduced in \cite{tirer2022perturbation} as
\begin{align}\label{eq:D measure}
D_l  \;:=\;  \mathrm{Tr}(\bm{\Sigma}_W^l) / \mathrm{Tr}(\bm{\Sigma}_B^l ),
\end{align}
where $\bm{\Sigma}_W^l$ and $\bm{\Sigma}_B^l$ characterize the between-class and with-class variabilities for the $l$-th layer respectively as
\begin{align}\label{eq:Sigma}
   \bm{\Sigma}_W^l   = \frac{1}{N} \sum_{k=1}^K \sum_{i=1}^{n} \left( \bz_{k,i}^l - \bar{\bz}_k^l \right)\left( \bz_{k,i}^l - \bar{\bz}_k^l \right)^\top,\   \bm{\Sigma}_B^l  = \frac{1}{K}\sum_{k=1}^K \left(  \bar{\bz}_k^l - \bar{\bz}^l \right)\left( \bar{\bz}_k^l - \bar{\bz}^l \right)^\top
\end{align}
for all $l \in [L-1]$, where $\bar{\bz}_k^l = \frac{1}{n} \sum_{i=1}^{n} \bz_{i,k}^l$ denotes the sample of the $k$-th class for the $l$-th layer's feature, and $\bar{\bz}^l = \frac{1}{K} \sum_{k=1}^{K} \bar{\bz}_k^l$ denotes  the corresponding global sample mean. Intuitively, $\mathrm{Tr}(\bm{\Sigma}_W^l)$ measures how well the features $\bm{z}_{k,i}^l$ in the $l$-th layer collapse to their means in each class, and $\mathrm{Tr}(\bm{\Sigma}_B)$ measures the discrimination between classes. 

Therefore, by using $\mathrm{Tr}(\bm{\Sigma}_B)$ as a normalization factor, the metric $D_l$ measures how the data are separated and concentrated to the class means in the $l$-th layer. The smaller the value of $D_l$ is, the more collapsed the features are. Additionally, it should be noted that the metric $D_l$ can be viewed as a simplification of the original metric $\mathrm{Tr}(\bm{\Sigma}_W^l (\bm{\Sigma}_B^l)^{\dagger})$ extensively studied in \cite{papyan2020prevalence,zhu2021geometric,he2022law}, by replacing the pseudoinverse with a trace division.

\paragraph{Theoretical Result: Progressive Data Separation with Linear Decay.} Based upon the above setup, we are ready to theoretically justify the progressive data separation phenomenon in \Cref{fig:progressive NC}, where the metric $D_l$ decays at least with a linear rate across layers. Our result is based on DLNs under some mild conditions. We defer all proofs to \Cref{app:proof4}. 

\begin{figure}[t]
    \begin{center}
    \includegraphics[width=0.9\linewidth]{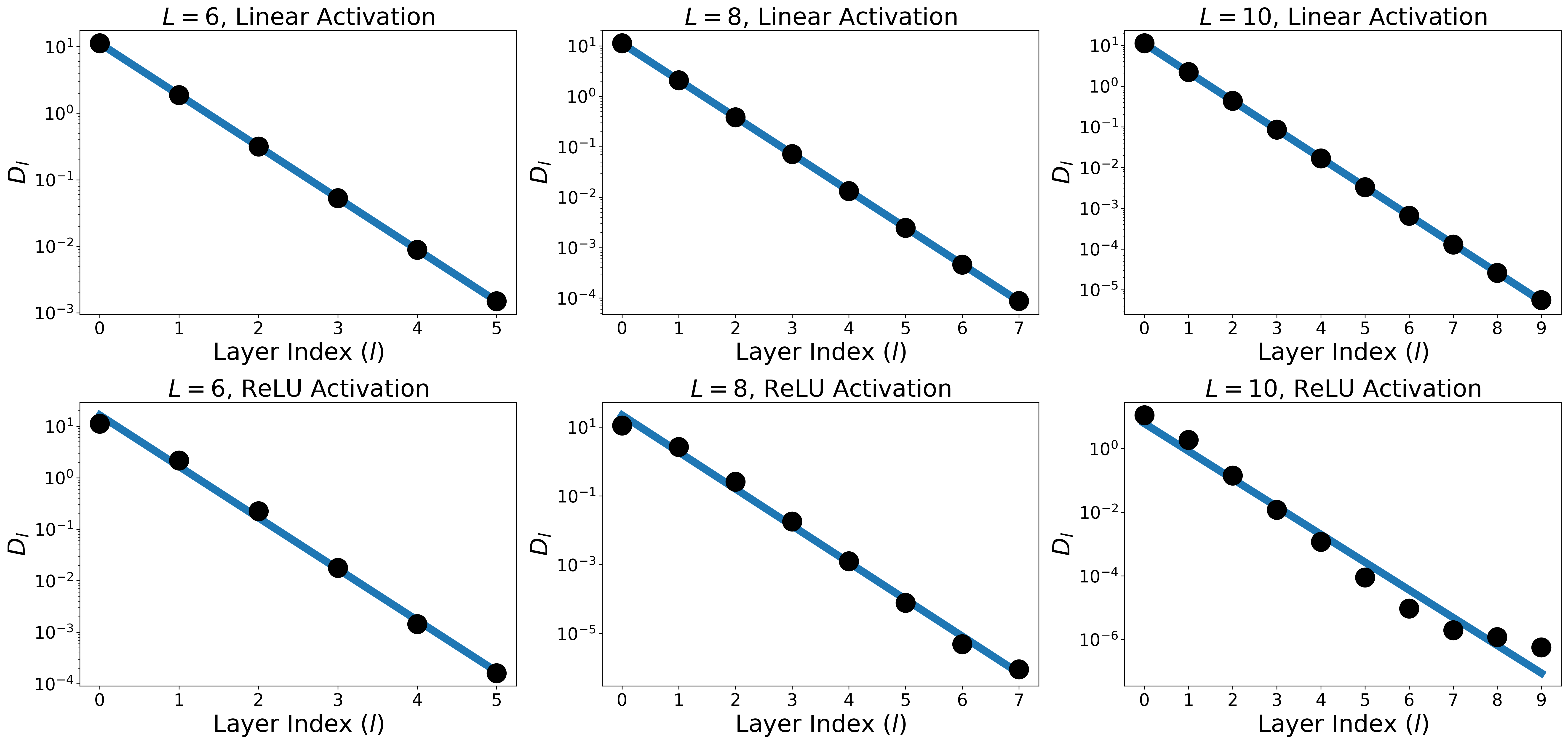}
    \end{center}
    \vspace{-0.1in}
    \caption{\textbf{Linear decay of feature separation in trained deep networks.} Varying the depth $L$ and activation type (Linear vs ReLU), we plot the separation measure $D_l$ in \eqref{eq:D measure} at each layer $l$ up to the penultimate layer ($L-1$), along with the best log-linear fit line. }
    \label{fig:grid_collapse}
    \vspace{-0.1in}
\end{figure}

\begin{theorem}\label{thm:NC}
For a $K$-class classification problem on a balanced dataset, suppose that the input dataset $\bm X \in \R^{d\times N}$ is square and orthogonal. For an $L$-layer DLN with parameters $\bm{\Theta} = \{\bW_l\}_{l=1}^L$ described in \eqref{eq:DLN} with $d_l = d > 2K$ for all $l \in [L-1]$, suppose that $\bm{\Theta}$ satisfies \\
(i)  Global Optimality: $\bW_{L:1}\bm{X} = \bm{Y}.$ \\
(ii) Balancedness: There exists a positive constant $\varepsilon \le \left\{\frac{n^{1/2L}}{\sqrt{30}L\sqrt[4]{d-K}},\frac{(n/2)^{1/4L}}{\sqrt[4]{d-K}}, \frac{1}{\sqrt{2(\sqrt{K}+1)}} \right\}$ such that 
$\bW_{l+1}^\top\bW_{l+1} = \bW_l\bW_l^\top, \forall l \in [L-2],\ \|\bW_{L}^\top\bW_{L} - \bW_{L-1}\bW_{L-1}^\top\|_F \le \varepsilon^2\sqrt{d-K}.$ \\ 
(iii)  Unchanged Spectrum: There exists a positive constant $\varepsilon > 0$ and an index set $\calA \subseteq [d]$ with $|\calA| = d-2K$ such that for all $l \in [L-1]$ that $\sigma_i(\bW_l) = \varepsilon,\ \forall i \in \calA.$ \\
Then, it holds for all $l=0,1,\dots,L-2$ that
\begin{align}\label{eq:linear decay}
  {D_{l+1}} / D_l  \le 2(\sqrt{K}+1)\varepsilon^2. 
\end{align}
\end{theorem}
\noindent Note that (iii) depends on \Cref{thm:main}. Based on the above theorem, we can directly verify that DLNs trained by gradient flow (i.e., $\eta \to 0_+$ in \eqref{eq:gd}) with proper initialization exhibit linear progressive collapse.  
\begin{coro}\label{coro:NC}
Consider the setting in \Cref{thm:NC}. Suppose that we employ gradient flow to train Problem \eqref{eq:obj}, and initialize $\mb W_l(0)$ for all $l \in [L-1]$ to be arbitrary $\varepsilon$-scaled orthogonal matrices satisfying \eqref{eq:init} and $\mb W_{L}(0) = \varepsilon[\mb U_L\ \mb 0]$, where $\mb U_L \in \calO^K$ and $\varepsilon \le \left\{\frac{n^{1/2L}}{\sqrt{30}L\sqrt[4]{d-K}},\frac{(n/2)^{1/4L}}{\sqrt[4]{d-K}}, \frac{1}{\sqrt{2(\sqrt{K}+1)}} \right\}$. If an optimal solution $\bm{\Theta}^*$ is found, then \eqref{eq:linear decay} holds.  
\end{coro}

\vspace{-0.1in}
\paragraph{Experimental Results.} 

We conduct numerical experiments to verify our theoretical results. In the multi-class classification problem, we consider $K=5$ classes and $n=10$ samples per class, and randomly generate the input data $\mb X \in \R^{d\times N}$ with $d = N = 50$ such that $\mb X \mb X^\top = \mb I$. We draw the network weights from random $\varepsilon$-scale orthogonal matrices with $\varepsilon = 0.5$, and vary the depth $L \in \{6, 8, 10\}$ as well as activation $\sigma \in \{\mbox{Linear}, \mbox{ReLU}\}$. The results are shown in \Cref{fig:grid_collapse}. First, we can see that the linear progressive decay phenomenon persists across networks of different depths and even in nonlinear networks. Second, as suggested by \eqref{eq:D measure}, the decay rate shown in \Cref{fig:grid_collapse} is \textit{independent} of the depth $L$, suggesting that deeper networks yield greater feature separation at the penultimate layer. Additional experimental results can be found in \Cref{app:addexp}.



%% file: main_sections/conclusion.tex
This paper offers a comprehensive analysis of the law of parsimony in gradient descent for learning DLNs, contributing to the ongoing pursuit of more efficient and effective deep-learning models. By uncovering the mechanisms that drive parsimonious solutions, we hope to inspire future research and the development of advanced techniques that harness the power of simplicity in deep learning. Ultimately, our goal is to bridge the gap between theory and practice, enabling practitioners to design and train deep learning models with improved efficiency and effectiveness. Finally, we conclude by providing a survey of related works and subsequently discuss the connections and distinctions between our results and the existing literature.  


\paragraph{Linear networks.} Due to their relative simplicity, deep linear networks are widely used as an alternative approach for investigating the optimization, generalization, and representation characteristics of non-linear networks. Regarding optimization properties, previous works such as \cite{kawaguchi2016deep,lu2017depth} have studied the optimization landscape of deep linear networks. Some recent works \cite{arora2018convergence,arora2018optimization,eftekhari2020training} have established convergence guarantees of gradient descent for training deep linear networks. Many seminal works are devoted to explaining the generalization ability of deep networks via different approaches, such as sharpness \cite{neyshabur2017exploring}, neural tangent kernel \cite{jacot2018neural,vyas2022limitations}, and implicit bias \cite{gunasekar2018implicit,huh2021low,valle2018deep}. In regards to understanding deep representations, recent seminal works \cite{papyan2020prevalence,fang2021exploring,
zhu2021geometric} studied an intriguing phenomenon termed neural collapse, which is prevalent across different network architectures, datasets, and training losses during the terminal phase of training. Notably, many researchers have provided theoretical explanations for this phenomenon by assuming the unconstrained feature model \cite{yaras2022neural,zhu2021geometric,han2021neural,zhou2022optimization,zhou2022are,wang2022linear,kothapalli2023neural}. This assumption simplifies over-parameterized nonlinear networks into two-layer linear networks, enabling a deeper understanding of the underlying mechanisms.    

\paragraph{Implicit bias.} In recent years, {\em implicit bias} (a.k.a. {\em implicit regularization}) has played an important role in understanding the phenomenon that deep neural networks in the over-parameterized setting often generalize well even when trained without any explicit regularization. Numerous studies have been dedicated to unraveling the mysteries of implicit bias, investigating it from both theoretical and empirical perspectives. In particular, {\em simplicity bias}, {\em low-rank bias}, and {\em spectral bias} have been extensively explored in the literature. Simplicity bias refers to the tendency of gradient descent (GD) for training deep networks to learn only the simplest features over other useful
but more complex features \cite{shah2020pitfalls,gissin2019implicit,morwani2023simplicity}, e.g., it has been shown \cite{soudry2018implicit,nacson2019convergence,gunasekar2018implicit} that GD favors max-margin solutions in linear models for classification problems. Low-rank bias refers to the notion that deep networks trained by GD are biased toward low-rank solutions \cite{gunasekar2017implicit,arora2019implicit,huh2021low}. For instance, \cite{gunasekar2017implicit} showed that GD for solving matrix factorization is biased towards minimum nuclear norm solutions. The line of works studied the robustness with overparameterization with implicit bias \cite{Hu2020Simple,you2020robust,liu2022robust}. The works \cite{gidel2019implicit,arora2019implicit} demonstrated that adding depth to matrix factorization enhances an implicit tendency towards minimum nuclear norm solutions, leading to more accurate recovery. Recently, \cite{huh2021low} provided substantial empirical observations and concluded that deep networks exhibit an inductive bias towards solutions with lower effective ranks. As for spectral bias, it describes the phenomenon that the learning dynamics of deep networks tends to find low-frequency functions. \cite{cao2019towards,rahaman2019spectral} studied the spectral bias of deep networks using tools from Fourier analysis. 


\paragraph{Comparison to existing works.} Here, we would like to highlight the differences and connections between existing work and our own. First, most of the existing works analyze the implicit bias by investigating the dynamics of gradient flow \cite{arora2019implicit,gissin2019implicit,min2021explicit}, whereas our work directly focuses on the dynamics of GD by utilizing the low-dimensional structure of the cross-correlation matrix $\mb Y \mb X^\top$. Our main proof idea is to show that the gradient updates across all iterates are low rank and share a common nullspace throughout the entire training process. Since the weights are initialized orthogonally, the right singular vectors can be chosen arbitrarily to align with the nullspace of
gradient updates so that they remain fixed. Second, it is worth noting that we can incorporate weight decay regularization into our analysis. Unlike previous work on implicit bias \cite{min2022convergence,gissin2019implicit,arora2019implicit,vardi2021implicit} which did not explicitly consider weight decay, we carefully examine the effect of this regularization technique. In particular, when weight decay regularizer $\lambda > 0$ is applied, we observe that the singular value $\rho(t)$ tends to zero asymptotically as $t$ goes to infinity. This finding indicates that gradient descent with weight decay is also biased towards finding low-rank solutions. Third, compared to the dynamics analysis of SVD of the product matrix (i.e., $\mb W_{L:1}$) in \cite{arora2019implicit}, which characterizes its singular values and vectors through partial differential equations, \Cref{thm:main} directly characterizes the dynamics of SVD of each weight matrix (i.e., $\mb W_l$) via revealing the structures of singular values and vectors. More precisely, our results not only accurately predict the singular values of each weight matrix during the training process, as shown in \eqref{eq:rho 1} and \eqref{eq:rho 2}, but also demonstrate that the subspace corresponding to the singular values in $\rho(t)$ remains unchanged throughout the iterations. 

\section*{Acknowledgment} 
CY and QQ acknowledge support from U-M START \& PODS grants, NSF CAREER CCF-2143904, NSF CCF-2212066, and NSF CCF-2212326. QQ and PW also acknowledge support from ONR N00014-22-1-2529, an AWS AI Award, and a gift grant from KLA. PW and LB acknowledge support from DoE award DE-SC0022186, ARO YIP W911NF1910027, and NSF CAREER CCF-1845076. ZZ acknowledges support from NSF grant CCF-2240708. WH acknowledges support from the Google Research Scholar Program.

%% file: main_sections/app_main.tex
\newpage
\onecolumn
\par\noindent\rule{\textwidth}{1pt}
\begin{center}
{\Large \bf Appendix}
\end{center}
\vspace{-0.1in}
\par\noindent\rule{\textwidth}{1pt}

\noindent In the appendix, 
we present additional experiments in \Cref{app:addexp}, and provide complete proofs for the technical results of Sections \ref{sec:results} and \ref{sec:appli} in \Cref{app:proof3} and \Cref{app:proof4} respectively. Before we proceed, we introduce some further notation. For simplicity, let $\ell(t) = \ell(\bm{\Theta}(t))$.  We use $\sigma_{\max}(\bA)$ (or $\|\bA\|$), $\sigma_i(\bA)$, and $\sigma_{\min}(\bA)$ to denote the largest, the $i$-th largest, and the smallest singular values of a matrix $\bA$, respectively. Given weight matrices $\bW_1,\dots,\bW_L$, we denote $\bW_{j:i} := \bW_j\bW_{j-1}\dots\bW_i$ if $1 \le i \le j \le L$, and $\bW_{0:1} := \bm{I}$ and $\bW_{L+1:L}=\bm{I}$. We denote the Kronecker product by $\otimes$. 

\section{Additional Experiments}\label{app:addexp}
In this section, we present additional experimental results to supplement those presented in the main paper. We note that all experiments in this work were carried out on a single NVIDIA Tesla V100 GPU. 

\begin{figure}[h!]
    \centering
    \includegraphics[width=0.65\linewidth]{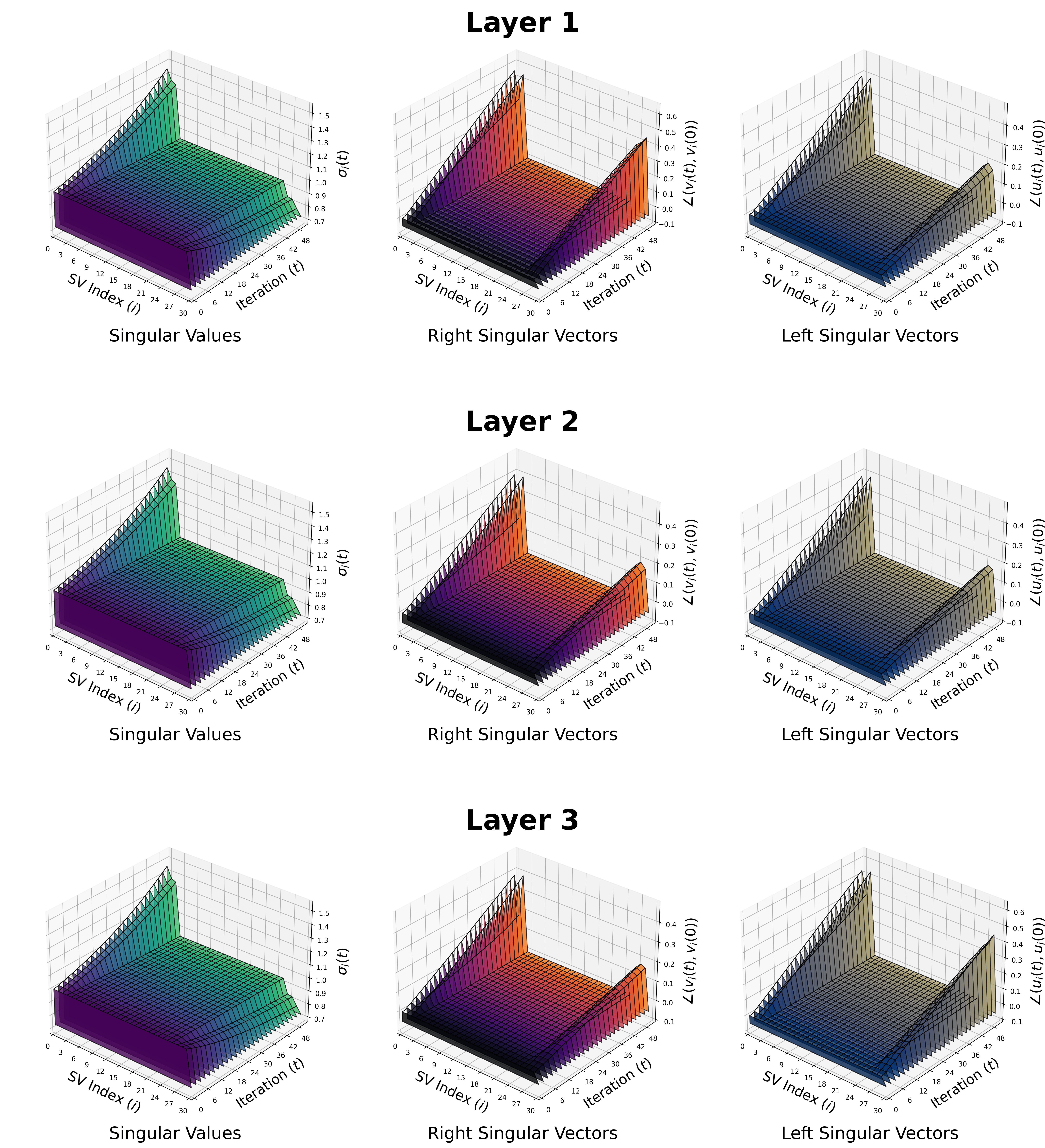}
    \caption{\textbf{Evolution of SVD of weight matrices for \textit{all} layers.} We visualize the SVD of the weight matrix for each layer of an $L=3$ layer deep linear network with $d_x = d_y = 30$ and $r=3$ (\textbf{Case 1}) throughout GD without weight decay. \textit{Left}: Magnitude  of the $i$-th singular value $\sigma_i(t)$ at iteration $t$. \textit{Middle}: Angle $\angle(\vv_i(t), \vv_i(0))$ between the $i$-th right singular vector at iteration $t$ and initialization. \textit{Right}: Angle $\angle(\vu_i(t), \vu_i(0))$ between the $i$-th left singular vector at iteration $t$ and initialization.
    }
    \label{fig:thm_sup}
\end{figure}

\begin{figure}[t]
    \centering
    \vspace{-3em}
    \includegraphics[width=1.0\linewidth]{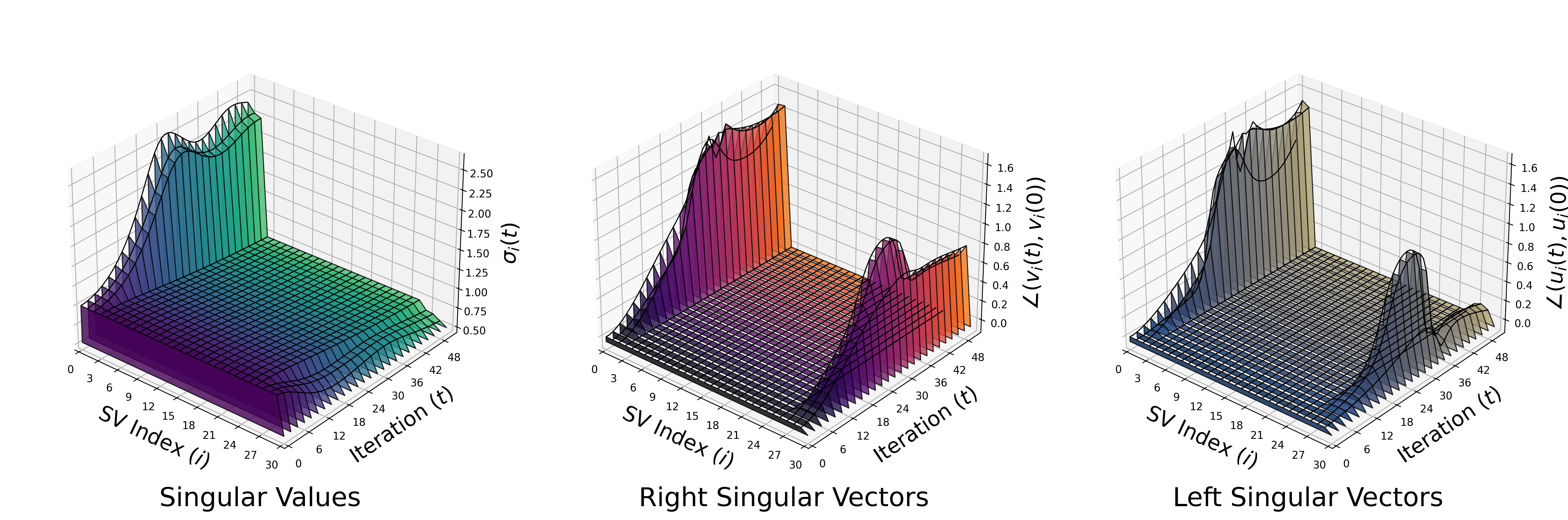}
    \caption{\textbf{Evolution of SVD of weight matrices \textit{with momentum}.} We visualize the SVD of the first layer weight matrix of an $L=3$ layer deep linear network with $d_x = d_y = 30$ and $r=3$ (\textbf{Case 1}) throughout GD \textit{with momentum} \eqref{eq:momentum}. \textit{Left}: Magnitude  of the $i$-th singular value $\sigma_i(t)$ at iteration $t$. \textit{Middle}: Angle $\angle(\vv_i(t), \vv_i(0))$ between the $i$-th right singular vector at iteration $t$ and initialization. \textit{Right}: Angle $\angle(\vu_i(t), \vu_i(0))$ between the $i$-th left singular vector at iteration $t$ and initialization.
    }
    \label{fig:momentum}
\end{figure}

First, we extend the visualization in \Cref{fig:traj_matrices} to \textit{all} layers, rather than simply the first layer -- the results are shown in \Cref{fig:thm_sup}. We see that the same parsimonious structures persist across all layers throughout GD, as implied by \Cref{thm:main}.

Next, we demonstrate that the law of parsimony generalizes to other variants of GD beyond weight decay, such as GD \textit{with momentum}, i.e., we update all weights for $t = 2,3,\dots$ as 
\begin{align}\label{eq:momentum}
    \bm{W}_l(t) = \bm{W}_l(t-1) - \eta \nabla_{\mW_l} \ell(\mTheta(t-1)) + \mu (\bm W_l(t-1) - \bm W_l(t-2))
\end{align}
for all $l \in [L]$, where $0 \leq \mu < 1$ is the momentum parameter and $\eta>0$ is the learning rate. We see in \Cref{fig:momentum} that introducing the momentum term still results in a low-dimensional trajectory along minimal singular subspaces. Using a similar approach to the analysis presented in this work, it should be fairly straightforward to extend our results to GD with momentum, but this is left as future work.

\begin{figure}[h!]
    \centering
    \includegraphics[width=1.0\linewidth]{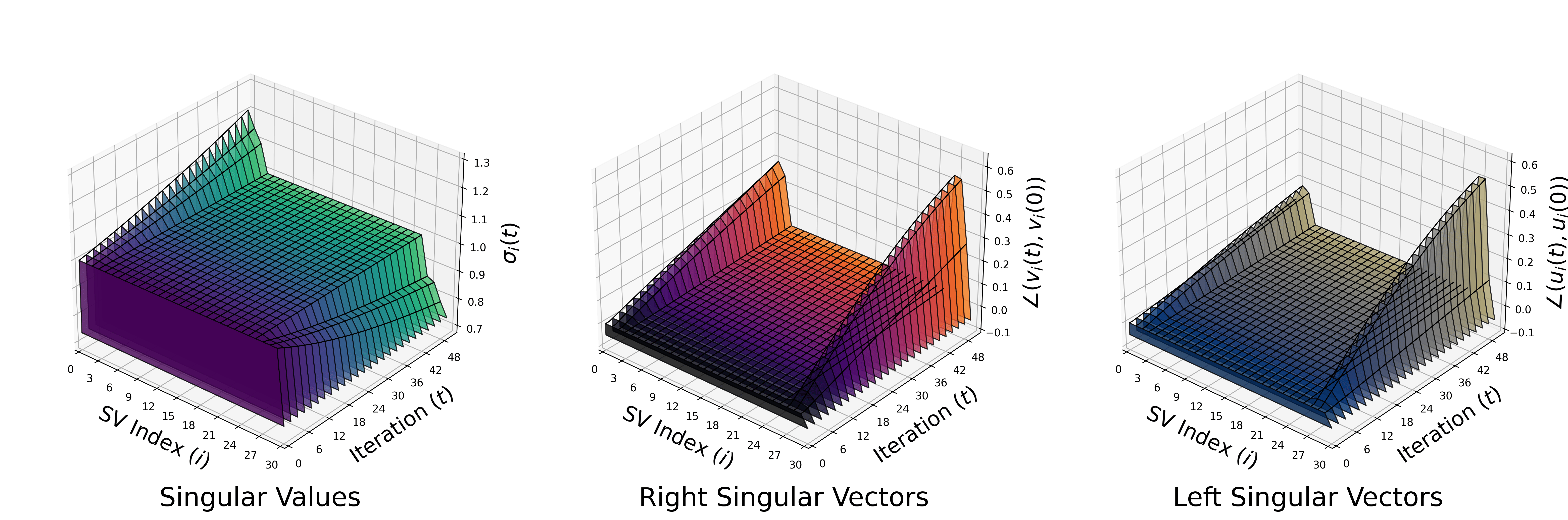}
    \caption{\textbf{Evolution of SVD of weight matrices \textit{without whitened data}.} We visualize the SVD of the first layer weight matrix of an $L=3$ layer deep linear network with $d_x = 30$, $d_y = 3$, and $N = 300$ throughout GD without weight decay. \textit{Left}: Magnitude  of the $i$-th singular value $\sigma_i(t)$ at iteration $t$. \textit{Middle}: Angle $\angle(\vv_i(t), \vv_i(0))$ between the $i$-th right singular vector at iteration $t$ and initialization. \textit{Right}: Angle $\angle(\vu_i(t), \vu_i(0))$ between the $i$-th left singular vector at iteration $t$ and initialization.
    }
    \label{fig:nonwhitened}
\end{figure}

Besides the low-dimensional structure of $\bm Y \bm X^\top$, the other data assumption made in this work is that the input data is whitened, i.e., $\bm X \bm X^\top = \bm I$. Here, we demonstrate that the law of parsimony continues to hold \textit{approximately} when $\bm X$ is not whitened or preprocessed in any way -- in particular, we draw the entries of $\bm X \in \R^{d_x \times N}$ i.i.d. from the standard normal distribution. The results are shown in \Cref{fig:nonwhitened}. We can see that most of the singular value remain unchanged from initialization as with before, and the corresponding singular subspaces evolve very little throughout optimization.

To conclude this section, we present several extensions to the experimental results in \Cref{sec:separation}. First, to investigate the effect of initialization scale $\varepsilon$ on the rate of linear progressive collapse, we fix the depth $L=4$ and vary $\varepsilon \in \{0.5, 0.25, 0.125\}$, keeping all other aspects of the setup the same as in \Cref{sec:separation}. The results are shown in \Cref{fig:init_scale_collapse}. As implied by \eqref{eq:linear decay}, we see that decreasing the initialization scale $\varepsilon$ leads to a steeper decay in the measure $D_l$ \eqref{eq:D measure} in both linear and nonlinear networks, demonstrating that there is indeed a necessary dependence on $\varepsilon$ in the upper bound \eqref{eq:linear decay}. Therefore, the initialization scale should be chosen carefully when training deep networks in practice, particularly due to the real implications of progressive collapse in transfer learning \cite{li2022principled}.

\begin{figure}[t]
    \centering
    \includegraphics[width=0.8\linewidth]{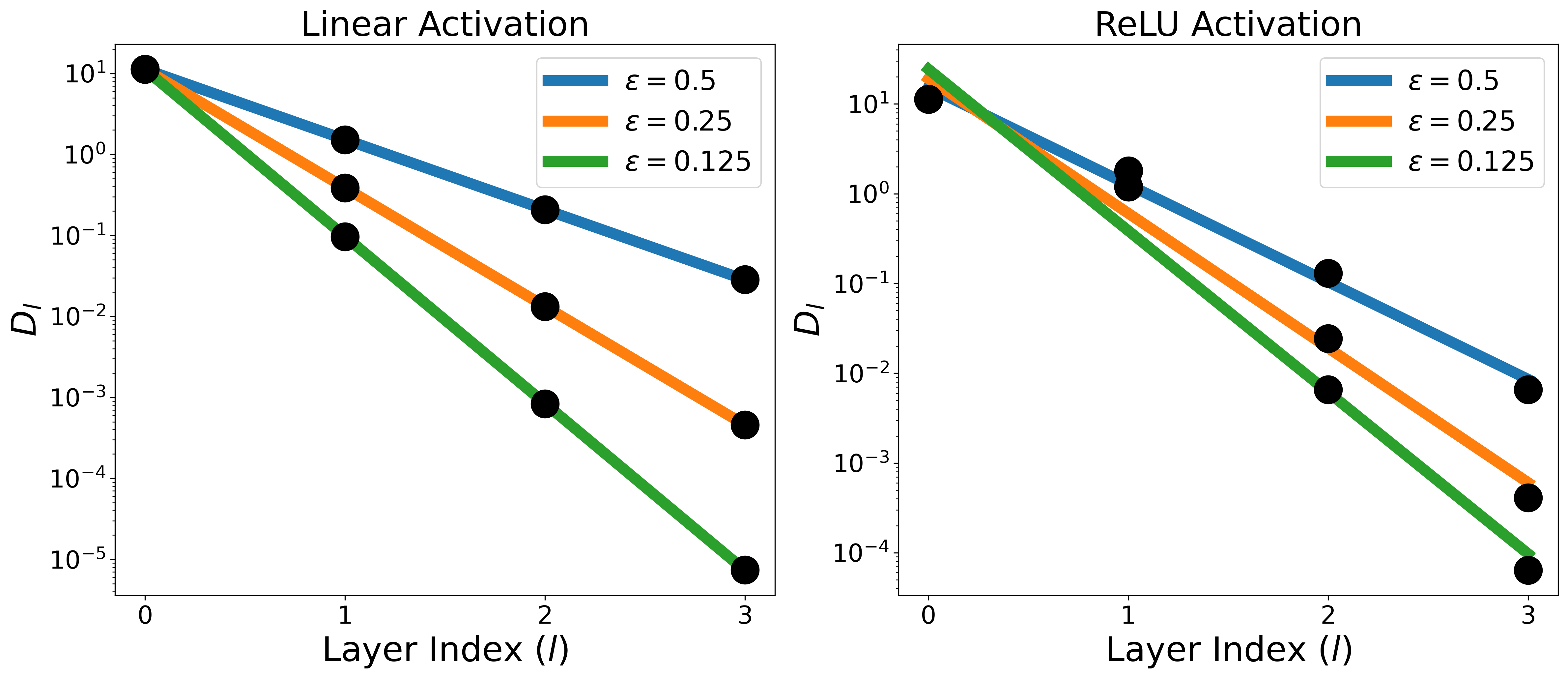}
    \caption{\textbf{Linear decay of feature separation in trained deep networks with varying initialization scale $\varepsilon$.} Varying the initialization scale $\varepsilon$ for networks with both linear and ReLU activations for an $L=4$ deep network, we plot the separation measure $D_l$ in \eqref{eq:D measure} at each layer $l$ up to the penultimate layer ($L-1$), along with the best log-linear fit line.}
    \label{fig:init_scale_collapse}
\end{figure}

Throughout this work, we have assumed that the network weights are initialized \textit{orthogonally} for the sake of analysis -- we now empirically verify that the linear progressive collapse phenomenon is \textit{independent} of the type of initialization, i.e., the distribution from which the weight matrices are drawn. With the same experimental set-up as in \Cref{sec:separation}, we fix the depth $L=8$ and vary the initialization type among orthogonal, normal, and uniform initializations. The results are shown in \Cref{fig:init_type_collapse}. To keep the initialization scale comparable between different initialization schemes, for the normally distributed weights, we draw the entries independently from $\mathcal{N}\left(0, \frac{\varepsilon^2}{d}\right)$, and for the uniformly distributed weights, we draw the entries independently from $\mathcal{U}\left(-\frac{3\varepsilon}{d}, \frac{3\varepsilon}{d}\right)$. This guarantees that drawing $\mW \in \R^{d \times d}$ from any distribution gives $\E[\|\mW\|_F^2] = d\cdot \varepsilon^2$. As a result, we see that progressive collapse occurs in both linear and nonlinear networks, and in fact decays at the same rate, regardless of the kind of initialization used. 

\begin{figure}[h!]
    \begin{center}
    \includegraphics[width=1.0\linewidth]{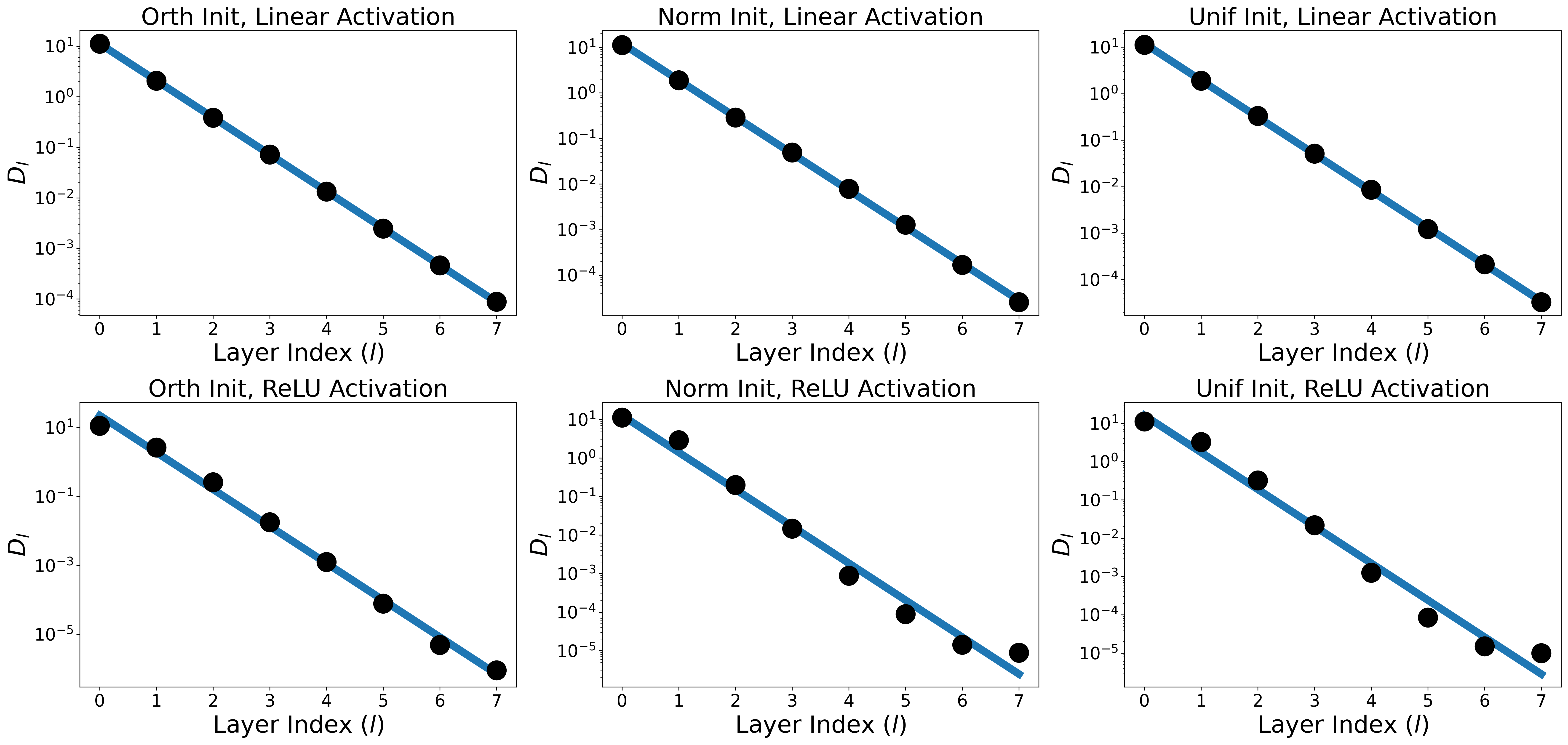}
    \end{center}
    \caption{\textbf{Linear decay of feature separation in trained deep networks with different initialization types.} Varying the initialization type (orthogonal vs normal vs uniform) and activation type (Linear vs ReLU) for an $L=8$ deep network, we plot the separation measure $D_l$ in \eqref{eq:D measure} at each layer $l$ up to the penultimate layer ($L-1$), along with the best log-linear fit line. }
    \label{fig:init_type_collapse}
\end{figure}

\section{Proofs in \Cref{sec:results}}\label{app:proof3}
In this section, we work towards proving \Cref{thm:main} by establishing Lemmas \ref{lem:r} and \ref{lem:dy}, which imply the conclusions of \textbf{Case 1} and \textbf{Case 2} respectively. Before proceeding, we note that by the assumption that the data $\bm{X} \in \R^{d_x \times n}$ is whitened, i.e., $\bm{X}\bm{X}^\top = \bm{I}_{d_x}$, we have that ${\bm \Gamma}(t)$ in \eqref{eq:gd} takes the form
\begin{equation}\label{eq:gam}
    {\bm \Gamma}(t) = \bW_{L:1}(t) - \bm{\Phi},\ \text{where}\  \bm{\Phi} = \bm{Y}\bm{X}^\top \in \R^{d_y \times d_x}
\end{equation}
In addition, we note that all statements quantified by $i$ in this section implicity hold for all $i \in [m]$ (as defined in \Cref{thm:main}) for the sake of notational brevity. 



\begin{lemma}\label{lem:r}
Under the setting of \Cref{thm:main} (\textbf{Case 1}), there exist orthonormal sets $\{\bm u_i^{(l)}\}_{i=1}^m \subset \R^d$ and $\{\bm v_i^{(l)}\}_{i=1}^m \subset \R^d$ for $l \in [L]$ satisfying $\bm v_i^{(l+1)} = \bm u_i^{(l)}$ for all $l \in [L-1]$ such that the following hold for all $t \geq 0$: 
\begin{alignat*}{2}
    \mathcal{A}(t)&: \bm W_{l}(t) \bm v_i^{(l)} = \rho(t) \bm u_i^{(l)} \quad &&\forall l \in [L], \\
    \mathcal{B}(t)&: \bm W_{l}^\top(t) \bm u_i^{(l)} = \rho(t) \bm v_i^{(l)} \quad &&\forall l \in [L], \\
    \mathcal{C}(t)&: \bm \Phi^\top \bm W_{L:l+1}(t) \bm u_i^{(l)} = \bm 0 \quad && \forall l \in [L], \\
    \mathcal{D}(t)&: \bm \Phi \bm W^\top_{l-1:1}(t)\bm v_i^{(l)} = \bm 0 \quad && \forall l \in [L],
\end{alignat*}
where $\rho(t) = \rho(t-1)\left(1 - \eta\lambda -\eta \cdot \rho(t-1)^{2(L-1)}\right)$ for all $t \geq 1$ with $\rho(0) = \varepsilon$.
\end{lemma}

\begin{proof}
Define $\bm \Psi := \bm W^\top_{L:2}(0)\bm \Phi$. Since the rank of $\bm \Phi$ is at most $r$, we have that the rank of $\bm \Psi \in \R^{d \times d}$ is at most $r$, which implies that $\mathrm{dim} \,\calN\left(\bm \Psi\right) = \mathrm{dim}\,\calN\left(\bm \Psi^\top \right) \ge d - r$. We define the subspace
\begin{align*}
    \calS := \calN\left(\bm \Psi\right)\cap \calN\left(\bm \Psi^\top \bm{W}_1(0)\right) \subset \R^d.
\end{align*}
Since $\bm{W}_1(0) \in \R^{d\times d}$ is nonsingular, we have
\begin{align*}
    \mathrm{dim}(\calS) \ge 2(d-r) - d = m. 
\end{align*}
Let $\{\bm{v}_i^{(1)}\}_{i=1}^m$ denote an orthonormal basis for $\calS$ and set $\bm{u}_i^{(1)} := \bW_1(0)\bv_i^{(1)}/\varepsilon$, where $\varepsilon > 0$ is the constant in \eqref{eq:init} -- since $\mW_1(0)$ is orthogonal, $\{\bm u_i^{(1)}\}_{i=1}^m$ is also an orthonormal set. Then we trivially have $\bW_1(0)\bv_i^{(1)} = \varepsilon\bm{u}_i^{(1)}$, which, together with \eqref{eq:init}, implies $\bm W_{1}^\top(0) \bm u_i^{(1)} = \varepsilon \bm v_i^{(1)}$. It follows from $\bm{v}_i^{(1)} \in \calS$ that $\bm \Psi \bm{v}_i^{(1)} = \bm{0}$ and $\bm \Psi^\top \bW_1(0) \bm{v}_i^{(1)} = \bm{0}$, which is equivalent to $\bm W^\top_{L:2}(0)\bm \Phi \bm v_i^{(1)} = \bm 0$ and $\bm\Phi^\top \bm W_{L:2}(0) \bm W_1(0) \bm v_i^{(1)} = \varepsilon \bm\Phi^\top \bm W_{L:2}(0) \bm u_i^{(1)} = \bm 0$ respectively. Since $\mW_{L:2}^\top(0)$ is full column rank, we further have that $\bm \Phi \bm v_i^{(1)} = \bm 0$.

Now let $\mathcal E(l)$ be the event that we have orthonormal sets $\{\bm u_i^{(l)}\}_{i=1}^m$ and $\{\bm v_i^{(l)}\}_{i=1}^m$ satisfying $\bm{W}_{l}(0) \bm v_i^{(l)} = \varepsilon \bm u_i^{(l)}$, $\bm W_{l}^\top(0) \bm u_i^{(l)} = \varepsilon \bm v_i^{(l)}$, $\bm \Phi^\top \bm W_{L:l+1}(0) \bm u_i^{(l)} = \bm 0$, and $\bm \Phi \bW_{l-1:1}^\top(0) \bm v_i^{(l)} = \bm 0$. From the above arguments, we have that $\mathcal E(1)$ holds -- now suppose $\mathcal E(k)$ holds for some $1 \le k < L$. Set $\bm v_i^{(k+1)} := \bm u_i^{(k)}$ and $\bm u_i^{(k+1)} := \bm W_{k+1}(0) \bm v_i^{(k+1)}/\varepsilon$. This, together with \eqref{eq:init}, implies that $\bm W_{k+1}(0) \bm v_i^{(k+1)} = \varepsilon \bm u_i^{(k+1)}$ and $\bm W_{k+1}^\top(0) \bm u_i^{(k+1)} = \varepsilon \bm v_i^{(k+1)}$. Moreover, we have 
\begin{align*}
    \bm \Phi^\top \bm W_{L:(k+1)+1}(0) \bm u_i^{(k+1)} &= \bm \Phi^\top\bm W_{L:k+1}(0) \bm W_{k+1}^\top(0) \bm u_i^{(k+1)} / \varepsilon^2 \\
    &= \bm \Phi^\top\bm W_{L:k+1}(0) \bm v_i^{(k+1)} / \varepsilon \\
    &= \bm \Phi^\top\bm W_{L:k+1}(0) \bm u_i^{(k)} / \varepsilon  = \bm 0,
\end{align*}
where the first two equalities follow from \eqref{eq:init} and  $\bm u_i^{(k+1)} = \bm W_{k+1}(0) \bm v_i^{(k+1)}/\varepsilon$, and the last equality is due to $\bm v_i^{(k+1)} = \bm u_i^{(k)}$. Similarly, we have
\begin{align*}
    \bm \Phi \bm W_{(k+1)-1:1}^\top(0) \bm v_i^{(k+1)} & = \bm \Phi \bm W_{k-1:1}^\top(0) \bm W_k^\top(0) \bm v_i^{(k+1)}  \\
    &= \bm \Phi \bm W_{k-1:1}^\top(0)  \bm W_k^\top(0) \bm u_i^{(k)} \\
    & = \varepsilon \bm \Phi \bm W_{k-1:1}^\top(0) \bm v_i^{(k)} = \bm 0,
\end{align*}
where the second equality follows from $\bm v_i^{(k+1)} = \bm u_i^{(k)}$ and the third equality is due to $\bm W_k^\top(0)\bm u_i^{(k)} = \varepsilon \bm v_i^{(k)}$. Therefore $\calE(k+1)$ holds, so we have $\calE(l)$ for all $l \in [L]$. As a result, we have shown the base cases $\calA(0)$, $\calB(0)$, $\calC(0)$, and $\calD(0)$. 

Now we proceed by induction on $t \ge 0$. Suppose that $\calA(t)$, $\calB(t)$, $\calC(t)$, and $\calD(t)$ hold for some $t \ge 0$. First, we show $\calA(t+1)$ and $\calB(t+1)$. We have
\begin{align*}
    \mW_l(t+1) \bm v_i^{(l)}  & =  \left[(1-\eta\lambda) \bm W_l(t) -  \eta \bm W_{L:l+1}^\top (t) \bm \Gamma (t)\bm W_{l-1:1}^\top (t) \right]\bm v_i^{(l)} \\
    &= \left[(1-\eta\lambda) \bm W_l(t) -  \eta \bm W_{L:l+1}^\top (t)\left(\bm W_{L:1}(t) - \bm \Phi\right)\bm W_{l-1:1}^\top (t) \right]\bm v_i^{(l)} \\
    & = (1-\eta\lambda)\bm{W}_l(t) \bm v_i^{(l)} - \eta \bm W_{L:l+1}^\top(t) \bm W_{L:1}(t) \bm W_{l-1:1}^\top(t) \bm v_i^{(l)} \\
    & = (1-\eta\lambda)\bm{W}_l(t) \bm v_i^{(l)} - \eta \cdot \rho(t)^{2(L-1)} \bm{W}_l(t) \bm v_i^{(l)} \\
    &= \rho(t) \left(1 - \eta \lambda - \eta \cdot \rho(t)^{2(L-1)}\right) \bm u_i^{(l)} = \rho(t+1) \bm u_i^{(l)}
\end{align*}
for all $l \in [L]$, where the first equality follows from \eqref{eq:gd}, the second equality follows from \eqref{eq:gam}, the third equality follows from $\mathcal{D}(t)$, and the fourth equality follows from $\mathcal{A}(t)$ and $\mathcal{B}(t)$ applied repeatedly along with $\bm v_i^{(l+1)} = \bm u_i^{(l)}$ for all $l \in [L-1]$, proving $\mathcal{A}(t+1)$. Similarly, we have 
\begin{align*}
    \mW_l^\top(t+1) \bm u_i^{(l)}  & = \left[(1-\eta\lambda) \bm W_l^\top(t) -  \eta \bm W_{l-1:1}(t) \bm \Gamma^\top(t) \bm W_{L:l+1}(t) \right]\bm u_i^{(l)} \\
    &= \left[(1-\eta\lambda) \bm W_l^\top(t) -  \eta \bm W_{l-1:1}(t) \bm \left(\bm W_{L:1}^\top(t) - \bm \Phi^\top\right) \bm W_{L:l+1}(t) \right]\bm u_i^{(l)} \\
    &= (1-\eta\lambda) \bm W_l^\top(t) \bm u_i^{(l)} - \eta \bm W_{l-1:1}(t) \bm W_{L:1}^\top(t) \bm W_{L:l+1}(t) \bm u_i^{(l)} \\
    &= (1-\eta\lambda) \bm W_l^\top(t) \bm u_i^{(l)} - \eta \cdot \rho(t)^{2(L-1)} \bm W_l^\top(t) \bm u_i^{(l)} \\
    &= \rho(t) \left(1 - \eta \lambda - \eta \cdot \rho(t)^{2(L-1)}\right) \bm v_i^{(l)} = \rho(t+1) \bm v_i^{(l)}
\end{align*}
for all $l \in [L]$, where the third equality follows from $\mathcal{C}(t)$, and the fourth equality follows from $\mathcal{A}(t)$ and $\mathcal{B}(t)$ applied repeatedly along with $\bm v_i^{(l+1)} = \bm u_i^{(l)}$ for all $l \in [L-1]$, proving $\mathcal{B}(t+1)$. Now, we show $\calC(t+1)$. For any $k \in [L-1]$, it follows from $\bv_i^{(k+1)} = \bu_i^{(k)}$ and $\calA(t+1)$ that 
\begin{equation*}
    \bm W_{k+1}(t+1) \bm u_i^{(k)} = \bm W_{k+1}(t+1) \bm v_i^{(k+1)}  = \rho(t+1) \bm u_i^{(k+1)}.
\end{equation*}
Repeatedly applying the above equality for $k = l, l+1, \dots, L-1$, we obtain
\begin{equation*}
          \bm \Phi^\top \bm W_{L:l+1}(t) \bm u_i^{(l)} = \rho(t+1)^{L-l} \bm \Phi^\top \bm u_i^{(L)} = \bm 0
\end{equation*}
which follows from $\mathcal{C}(t)$, proving $\mathcal{C}(t+1)$. Finally, we show $\calD(t+1)$. For any $k \in \{2, \dots, L\}$, it follows from $\bm v_i^{(k)} = \bm u_i^{(k-1)}$ and $\calB(t+1)$ that 
\begin{equation*}
    \bm W_{k-1}^\top(t+1) \bm v_i^{(k)} = \bm W_{k-1}^\top(t+1) \bm u_i^{(k-1)} = \rho(t+1) \bm v_i^{(k-1)}.  
\end{equation*}
Repeatedly applying the above equality for $k = l, l-1, \dots, 2$, we obtain
\begin{equation*}
    \bm \Phi \bm W^\top_{l-1:1}(t)\bm v_i^{(l)} = \bm \Phi \bm v_i^{(1)} = \bm 0
\end{equation*}
which follows from $\mathcal{D}(t)$. Thus we have proven $\mathcal D(t+1)$, concluding the proof. 
\end{proof}

\begin{lemma}\label{lem:dy}
Under the setting of \Cref{thm:main} (\textbf{Case 2}), there exist orthonormal sets $\{\bm u_i^{(l)}\}_{i=1}^m \subset \R^d$ for $l \in [L-1]$ and $\{\bm v_i^{(l)}\}_{i=1}^m \subset \R^d$ for $l \in [L]$ satisfying $\bm v_i^{(l+1)} = \bm u_i^{(l)}$ for all $l \in [L-1]$ such that the following hold for all $t \geq 0$: 
\begin{alignat*}{2}
    \mathcal{A}(t)&: \bm W_{l}(t) \bm v_i^{(l)} = \rho(t) \bm u_i^{(l)} \quad &&\forall l \in [L-1], \\
    \mathcal{B}(t)&: \bm W_{l}^\top(t) \bm u_i^{(l)} = \rho(t) \bm v_i^{(l)} \quad 
 &&\forall l \in [L-1],\\
    \mathcal{C}(t)&: \bm W_{L:l+1}(t)\bm u_i^{(l)} = \bm 0 \quad  && \forall l \in [L-1], \\
    \mathcal{D}(t)&: \bm \Gamma(t) \bm W_{l-1:1}^\top(t) \bm v_i^{(l)} = \bm 0 \quad  && \forall l \in [L],
\end{alignat*}
where $\rho(t) = \rho(t-1) (1 - \eta \lambda)$ for all $t\geq 1$ with $\rho(0) = \varepsilon$.
\end{lemma}
\begin{proof}
Since $\bm{\Gamma}(0) \in \R^{d_y\times d_x}$, the rank of
\begin{align}\label{eq1:lem d=dx}
\nabla_{\bW_1}\ell(0) = \bW_{L:2}^\top(0)\bm{\Gamma}(0) \in \R^{d \times d}   
\end{align}
is at most $d_y$, which implies that $\mathrm{dim} \,\calN\left(\nabla_{\bW_1}\ell(0)\right) = \mathrm{dim}\,\calN\left(\nabla_{\bW_1}^\top\ell(0)\right) \ge d - d_y$. We define the subspace
\begin{align*}
    \calS := \calN\left(\nabla_{\bW_1}\ell(0)\right)\cap \calN\left(\nabla_{\bW_1}^\top\ell(0)\bm{W}_1(0)\right) \subset \R^d.
\end{align*}
Since $\bm{W}_1(0) \in \R^{d\times d}$ is nonsingular, we have
\begin{align*}
    \mathrm{dim}(\calS) \ge 2(d-d_y) - d = m. 
\end{align*}
Let $\{\bm{v}_i^{(1)}\}_{i=1}^m$ denote an orthonormal basis for $\calS$ and set $\bm{u}_i^{(1)} := \bW_1(0)\bv_i^{(1)}/\varepsilon$, where $\varepsilon > 0$ is the constant in \eqref{eq:init} -- since $\mW_1(0)$ is orthogonal, $\{\bm u_i^{(1)}\}_{i=1}^m$ is also an orthonormal set. Then we trivially have $\bW_1(0)\bv_i^{(1)} = \varepsilon\bm{u}_i^{(1)}$, which, together with \eqref{eq:init}, implies $\bm W_{1}^\top(0) \bm u_i^{(1)} = \varepsilon \bm v_i^{(1)}$. It follows from $\bm{v}_i^{(1)} \in \calS$ that $\nabla_{\bW_1}\ell(0) \bm{v}_i^{(1)} = \bm{0}$ and $\nabla_{\bW_1}^\top\ell(0)\bW_1(0) \bm{v}_i^{(1)} = \bm{0}$, which is equivalent to $\bW_{L:2}^\top(0) \bm{\Gamma}(0) \bm{v}_i^{(1)} = \bm{0}$ and $\bm{\Gamma}^\top(0) \bW_{L:2}(0) \bW_1(0)  \bm{v}_i^{(1)} = \varepsilon \bm{\Gamma}^\top(0) \bW_{L:2}(0) \bu_i^{(1)} = \bm{0}$ respectively by \eqref{eq1:lem d=dx}. Since $\bm{W}^\top_{L:2}(0)$ and $\bm{\Gamma}^\top(0)$ are full column rank, we have $\bm{\Gamma}(0) \bm{v}_i^{(1)} = \bm{0}$ and $\bW_{L:2}(0) \bu_i^{(1)} = \bm{0}$. It then follows that
\begin{align*}
    \bm{0} = \bm{\Gamma}(0)  \bm{v}_i^{(1)}  = \left(\bW_{L:1}(0) - \bm{\Phi}\right) \bm{v}_i^{(1)} = \varepsilon\bW_{L:2}(0)\bm{u}_i^{(1)} -  \bm{\Phi}\bm{v}_i^{(1)} = -\bm{\Phi}\bm{v}_i^{(1)}.  
\end{align*}
Therefore, we have $\bm{\Phi}\bm{v}_i^{(1)} = \bm{0}$. 

Now let $\mathcal E(l)$ be the event that we have orthonormal sets $\{\bm u_i^{(l)}\}_{i=1}^m$ and $\{\bm v_i^{(l)}\}_{i=1}^m$ satisfying $\bm{W}_{l}(0) \bm v_i^{(l)} = \varepsilon \bm u_i^{(l)}$, $\bm W_{l}^\top(0) \bm u_i^{(l)} = \varepsilon \bm v_i^{(l)}$, $\bm W_{L:l+1}(0) \bm u_i^{(l)} = \bm 0$, and $\bm \Gamma(0)\bW_{l-1:1}^\top(0) \bm v_i^{(l)} = \bm 0$. From the above arguments, we have that $\mathcal E(1)$ holds -- now suppose $\mathcal E(k)$ holds for some $1 \le k < L-1$. Set $\bm v_i^{(k+1)} := \bm u_i^{(k)}$ and $\bm u_i^{(k+1)} := \bm W_{k+1}(0) \bm v_i^{(k+1)}/\varepsilon$. This, together with \eqref{eq:init}, implies that $\bm W_{k+1}(0) \bm v_i^{(k+1)} = \varepsilon \bm u_i^{(k+1)}$ and $\bm W_{k+1}^\top(0) \bm u_i^{(k+1)} = \varepsilon \bm v_i^{(k+1)}$. Moreover, we have 
\begin{align*}
    \bm W_{L:(k+1)+1}(0) \bm u_i^{(k+1)} &=  \bm W_{L:k+1}(0) \bm W_{k+1}^\top(0) \bm u_i^{(k+1)} / \varepsilon^2 \\
    &= \bm W_{L:k+1}(0) \bm v_i^{(k+1)} / \varepsilon \\
    &= \bm W_{L:k+1}(0) \bm u_i^{(k)} / \varepsilon  = \bm 0,
\end{align*}
where the first two equalities follow from \eqref{eq:init} and  $\bm u_i^{(k+1)} = \bm W_{k+1}(0) \bm v_i^{(k+1)}/\varepsilon$, and the last equality is due to $\bm v_i^{(k+1)} = \bm u_i^{(k)}$. Similarly, we have
\begin{align*}
    \bm \Gamma(0) \bm W_{(k+1)-1:1}^\top(0) \bm v_i^{(k+1)} & = \bm \Gamma(0) \bm W_{k-1:1}^\top(0)  \bm W_k^\top(0) \bm v_i^{(k+1)}  \\
    &= \bm \Gamma(0) \bm W_{k-1:1}^\top(0)  \bm W_k^\top(0) \bm u_i^{(k)} \\
    & = \varepsilon \bm \Gamma(0) \bm W_{k-1:1}^\top(0) \bm v_i^{(k)} = \bm 0,
\end{align*}
where the second equality follows from $\bm v_i^{(k+1)} = \bm u_i^{(k)}$ and the third equality is due to $\bm W_k^\top(0)\bm u_i^{(k)} = \varepsilon \bm v_i^{(k)}$. Therefore $\calE(k+1)$ holds, so we have $\calE(l)$ for all $l \in [L-1]$. Finally, setting $\bm v_i^{(L)} = \bm u_i^{(L-1)}$, we have
\begin{align*}
    \bm \Gamma(0) \bm W_{L-1:1}^\top(0) \bm v_i^{(L)} & = \bm \Gamma(0) \bm W_{L-2:1}^\top(0) \bm W_{L-1}^\top(0) \bm v_i^{(L)}  \\
    &= \bm \Gamma(0) \bm W_{L-1:1}^\top(0)  \bm W_{L-1}^\top(0) \bm u_i^{(L-1)} \\
    & = \varepsilon \bm \Gamma(0) \bm W_{L-1:1}^\top(0)  \bm v_i^{(L-1)} = \bm 0.
\end{align*}
As a result, we have shown the base cases $\calA(0)$, $\calB(0)$, $\calC(0)$, and $\calD(0)$. 

Now we proceed by induction on $t \ge 0$. Suppose that $\calA(t)$, $\calB(t)$, $\calC(t)$, and $\calD(t)$ hold for some $t \ge 0$. First, we show $\calA(t+1)$ and $\calB(t+1)$. We have
\begin{align*}
    \mW_l(t+1) \bm v_i^{(l)}  & =  \left[(1-\eta\lambda) \bm W_l(t) -  \eta \bm W_{L:l+1}^\top (t){\bm \Gamma}(t)\bm W_{l-1:1}^\top (t) \right]\bm v_i^{(l)} \\
    & = (1-\eta\lambda)\bm{W}_l(t) \bm v_i^{(l)} = \rho(t) (1-\eta \lambda) \bm u_i^{(l)} = \rho(t+1) \bm u_i^{(l)} 
\end{align*}
for all $l \in [L-1]$, where the first equality follows from \eqref{eq:gd}, the second equality uses $\calD(t)$, and the last equality is due to $\calA(t)$, proving $\calA(t+1)$. Similarly, we have
\begin{align*}
    \bm W_{l}^\top(t+1) \bm u_i^{(l)} & = \left[(1-\eta\lambda)\bm W_{l}^\top(t) - \eta \bm W_{l-1:1}(t) \bm \Gamma^\top(t) \bm W_{L:l+1}(t)\right] \bm u_i^{(l)} \\
   & =(1-\eta\lambda)\bm W_{l}^\top(t)\bm u_i^{(l)} = \rho(t) (1-\eta\lambda) \bm v_i^{(l)} = \rho(t+1) \bm v_i^{(l)}
\end{align*}
for all $l \in [L-1]$, where the second equality uses $\calC(t)$, and the last equality is due to $\calB(t)$, proving $\calB(t+1)$. Now, we show $\calC(t+1)$. For any $k \in [L-2]$, it follows from $\bv_i^{(k+1)} = \bu_i^{(k)}$ and $\calA(t+1)$ that 
\begin{equation*}
    \bm W_{k+1}(t+1) \bm u_i^{(k)} = \bm W_{k+1}(t+1) \bm v_i^{(k+1)}  = \rho(t+1) \bm u_i^{(k+1)}.
\end{equation*}
Repeatedly applying the above equality for $k = l, l+1, \dots, L-2$, we obtain
\begin{align*}
          \bm W_{L:l+1}(t+1) \bm u_i^{(l)} & = \rho(t+1)^{L-l-1} \bm W_L(t+1) \bm u_i^{(L-1)} \\
          &= \rho(t+1)^{L-l-1} \left[ (1-\eta \lambda) \bm W_L(t) - \eta \bm \Gamma(t) \bm W_{L-1:1}^\top(t)\right] \bm u_i^{(L-1)} = \bm 0
\end{align*}
where $\bm W_L(t) \bm u_i^{(L-1)} = \bm 0$ follows from $\mathcal C(t)$ and $\bm \Gamma(t) \bm W_{L-1:1}^\top(t) \bm u_i^{(L-1)} = \bm 0$ follows from $\bm u_i^{(L-1)} = \bm v_i^{(L)}$ and $\mathcal D(t)$, proving $\mathcal C(t+1)$. Finally, we show $\calD(t+1)$. For any $k \in \{2, \dots, L\}$, it follows from $\bm v_i^{(k)} = \bm u_i^{(k-1)}$ and $\calB(t+1)$ that 
\begin{align*}
    \bm W_{k-1}^\top(t+1) \bm v_i^{(k)} = \bm W_{k-1}^\top(t+1) \bm u_i^{(k-1)} = \rho(t+1) \bm v_i^{(k-1)}.  
\end{align*}
Repeatedly applying the above equality for $k = l, l-1, \dots, 2$, we obtain
\begin{align*}
    \bm \Gamma(t+1) \bm W_{l-1:1}^\top(t+1) \bm v_i^{(l)} & = \rho(t+1)^{l-1}\bm \Gamma(t+1)\bm v_i^{(1)} \\
    & = \rho(t+1)^{l-1}(\bm W_{L:1}(t+1) - \bm \Phi) \bm v_i^{(1)} = \bm{0}
\end{align*}
where the last line follows from $\bm \Phi \bm v_i^{(1)} = \bm 0$ as well as $$\bm W_{L:1}(t+1) \bm v_i^{(1)} = \bm W_{L:2}(t+1) \bm W_1(t+1) \bm v_i^{(1)} = \rho(t+1) \bm W_{L:2}(t+1) \bm u_i^{(1)} = \bm 0$$ by $\mathcal{A}(t+1)$ and $\mathcal{C}(t+1)$. Thus we have proven $\mathcal D(t+1)$, concluding the proof. 
\end{proof}

\begin{proof}[Proof of \Cref{thm:main}]
We show the result for \textbf{Case 1} -- the proof is nearly identical for \textbf{Case 2}. 

By $\mathcal{A}(t)$ and $\mathcal{B}(t)$ of \Cref{lem:r}, there exists orthonormal matrices $\{ \bm U_{l, 2} \} \subset \R^{d \times m}$ and $\{ \bm V_{l, 2} \} \subset \R^{d \times m}$ for $l \in [L]$ satisfying $\bm U_{l+1, 2} = \bm V_{l, 2}$ for all $l \in [L-1]$ as well as 
\begin{equation}\label{eq:svd_22}
    \bm W_l(t) \bm V_{l, 2} = \rho(t) \bm U_{l, 2} \quad \mbox{and} \quad \bm W_l(t)^\top \bm U_{l, 2} = \rho(t) \bm V_{l, 2}
\end{equation}
for all $l \in [L]$ and $t \geq 0$, where $\rho(t)$ satisfies \eqref{eq:rho 1} for $t \geq 1$ with $\rho(0) = \varepsilon$. First, complete $\bm V_{1, 2}$ to an orthonormal basis for $\R^d$ as $\bm V_1 = \begin{bmatrix} \bm V_{1, 1} & \bm V_{1, 2} \end{bmatrix} \in \mathcal O^d$. Then for each $l \in [L-1]$, set $\bm U_l = \begin{bmatrix} \bm U_{l, 1} & \bm U_{l, 2} \end{bmatrix} \in \mathcal O^d$ where $\bm U_{l, 1} = \bm W_l(0) \bm V_{l, 1} / \varepsilon$ and $\bm V_{l+1} = \begin{bmatrix} \bm V_{l+1, 1} & \bm V_{l+1, 2} \end{bmatrix} \in \mathcal O^d$ where $\bm V_{l+1, 1} = \bm U_{l, 1}$, and finally set $\bm U_L = \begin{bmatrix} \bm U_{L, 1} & \bm U_{L, 2} \end{bmatrix} \in \mathcal O^d$ where $\bm U_{L, 1} = \bm W_L(0) \bm V_{L, 1} / \varepsilon$. We note that $\bm V_{l+1} = \bm U_l$ for each $l \in [L-1]$. Then we have
\begin{equation}\label{eq:svd_12}
    \bm U_{l, 1}^\top \bm W_l(t) \bm V_{l, 2} = \rho(t) \bm U_{l, 1}^\top \bm U_{l, 2} = \bm 0
\end{equation}
for all $l \in [L]$, where the first equality follows from \eqref{eq:svd_22}. Similarly, we also have
\begin{equation}\label{eq:svd_21}
    \bm U_{l, 2}^\top \bm W_l(t) \bm V_{l, 1} = \rho(t) \bm V_{l, 2}^\top \bm V_{l, 1} = \bm 0
\end{equation}
for all $l \in [L]$, where the first equality also follows from \eqref{eq:svd_22}. Therefore, combining \eqref{eq:svd_22}, \eqref{eq:svd_12}, and \eqref{eq:svd_21} yields
\begin{equation*}
    \bm U_l^\top \bm W_l(t) \bm V_l = \begin{bmatrix} \bm U_{l, 1} & \bm U_{l, 2} \end{bmatrix}^\top \bm W_l(t) \begin{bmatrix} \bm V_{l+1, 1} & \bm V_{l+1, 2} \end{bmatrix} = \begin{bmatrix}
        \wt{\bm W}_l(t) & \bm 0 \\ \bm 0 & \rho(t) \bm I_m
    \end{bmatrix}
\end{equation*}
for all $l \in [L]$, where $\wt{\bm W}_l(0) = \varepsilon \bm I_{2r}$ by construction of $\bm U_{l, 1}$. This directly implies \eqref{eq:weight_structures}, completing the proof.
\end{proof}

\section{Proofs in Section \ref{sec:appli}}\label{app:proof4} 

Suppose that $\bm{\Theta} = \{\mb W_l\}_{l=1}^L$ satisfies \\
(i)  Global Optimality: 
\begin{align}\label{eq:opti}
    \bW_{L:1}\bm{X} = \bm{Y}.
\end{align}
(ii) Balancedness: There exists a positive constant $\varepsilon \le \left\{\frac{n^{1/2L}}{\sqrt{30}L\sqrt[4]{d-K}},\frac{(n/2)^{1/4L}}{\sqrt[4]{d-K}}, \frac{1}{\sqrt{2(\sqrt{K}+1)}} \right\}$ such that 
\begin{align}\label{eq:bala}
    \bW_{l+1}^\top\bW_{l+1} = \bW_l\bW_l^\top, \forall l \in [L-2],\ \|\bW_{L}^\top\bW_{L} - \bW_{L-1}\bW_{L-1}^\top\|_F \le \varepsilon^2\sqrt{d-K}. 
\end{align}
(iii)  Unchanged Spectrum: There exists a positive constant $\varepsilon > 0$ and an index set $\calA \subseteq [d]$ with $|\calA| = d-2K$ such that for all $l \in [L-1]$ that
\begin{align}\label{eq:d-2K}
    \sigma_i(\bW_l) = \varepsilon,\ \forall i \in \calA.
\end{align}

For ease of exposition, we introduce some additional notation. In our analysis, we can assume $\bm{Y} = \bm{I}_K \otimes \bm{1}_n^\top$ without loss of generality. This, together with \eqref{eq:opti} and $\bm{X}$ is full rank, yields that the rank of $\bW_l$ is at least $K$ for all $l \in [L]$. Let
\begin{align}\label{eq:SVD Wl}
\bW_l = \bU_l\bm{\Sigma}_l\bV_l^\top =  \begin{bmatrix}
\bU_{l,1} & \bU_{l,2} 
\end{bmatrix}\begin{bmatrix}
\bm{\Sigma}_{l,1} & \bm{0} \\
\bm{0} & \bm{\Sigma}_{l,2}
\end{bmatrix}\begin{bmatrix}
\bV_{l,1}^\top \\ \bV_{l,2}^\top
\end{bmatrix} = \bU_{l,1} \bm{\Sigma}_{l,1}\bV_{l,1}^\top + \bU_{l,2}\bm{\Sigma}_{l,2}\bV_{l,2}^\top, 
\end{align}   
be a singular value decomposition (SVD) of $\bW_l$, where $\bm{\Sigma}_l \in \R^{d\times d}$ is diagonal, $\bm{\Sigma}_{l,1} = \diag\left(\sigma_{l,1},\dots,\sigma_{l,K} \right)$ with $\sigma_{l,1} \ge \cdots \ge \sigma_{l,K} > 0 $ being the singular values, and $\bm{\Sigma}_{l,2} = 
\diag\left(\sigma_{l,K+1},\dots,\sigma_{l,d} \right) $ with $\sigma_{l,K+1} \ge \cdots \ge \sigma_{l,d} \ge 0 $ being the remaining singular values of $\bW_l$; 
$\bU_{l} \in \calO^{d }$ with $\bU_{l,1} \in \R^{d\times K}$, $\bU_{l,2} \in \R^{d \times (d-K)}$; $\bV_{l} \in \calO^{d}$ with $\bV_{l,1} \in \R^{d \times K}$, $\bV_{l,2} \in \R^{d \times (d-K)}$. Noting that $\bW_L \in \R^{K\times d}$, let
\begin{align}\label{eq:SVD WL}
\bW_L =  \bU_L\bm{\Sigma}_L\bV_L^\top =  \bU_L \begin{bmatrix}
\bm{\Sigma}_{L,1} & \bm{0} 
\end{bmatrix}\begin{bmatrix}
\bV_{L,1}^\top \\ \bV_{L,2}^\top
\end{bmatrix} = \bU_{L}\bm{\Sigma}_{L,1}\bV_{L,1}^\top
\end{align}
be its singular value decomposition (SVD), where $\bm{\Sigma}_{L} \in \R^{K\times d}$, and $\bm{\Sigma}_{L,1} = \diag\left(\sigma_{L,1},\dots,\sigma_{L,K} \right)$ with $\sigma_{L,1} \ge \cdots \ge \sigma_{L,K} > 0 $ being the singular values; 
$\bU_{L} \in \calO^{K}$, and $\bV_{L} \in \calO^{d}$ with $\bV_{L,1} \in \R^{d\times K}$, $\bV_{L,2} \in \R^{d \times (d-K)}$.

Before we prove Theorem \ref{thm:NC}, we need some preliminary results. We first show that when $\{\mb W_l\}_{l=1}^L$ satisfy \eqref{eq:opti} and \eqref{eq:bala}, we can bound the leading $K$ singular values of $\mb W_l$ for all $l \in [L]$. Throughout this section, let 
\begin{align}\label{eq:delta}
\delta := \varepsilon^2\sqrt{d-K} \le \frac{n^{1/L}}{30L^2},
\end{align}
where the inequality follows from 
\begin{align}\label{eq:epsilon}
    \varepsilon \le \min\left\{\frac{n^{1/2L}}{\sqrt{30}L\sqrt[4]{d-K}},\frac{(n/2)^{1/4L}}{\sqrt[4]{d-K}}, \frac{1}{\sqrt{2(\sqrt{K}+1)}}\right\}. 
\end{align}

\begin{lemma}\label{lem:bound Wl}
Suppose that the weights $\{\bW_l\}_{l=1}^L$ satisfy \eqref{eq:opti} and \eqref{eq:bala}. 
Then, it holds that
\begin{align}\label{eq:bound Wl}
\left( \sqrt{\frac{n}{2}}\right)^{1/L} \le \sigma_K\left(\bW_l\right)  \le \sigma_1\left(\bW_l\right) \le \left( \sqrt{2n} \right)^{1/L},\ \forall l \in [L].
\end{align}
\end{lemma}
\begin{proof}[Proof of Lemma \ref{lem:bound Wl}] 
It follows from \eqref{eq:opti} and the fact that $\bm{X} \in \R^{d\times d}$ is orthogonal that $\bW_{L:1} = \bm{Y}\bm{X}^\top$. 
This, together with $\bY = \bI_K\otimes \bm{1}^\top_n$ and $\bm{X} \in \calO^d$, yields that  $\bW_{L:1}$ is of rank $K$ and
\begin{align}\label{eq1:lem bound W1}
   \sigma_i(\bW_{L:1}) = \sqrt{n},\ \forall i \in [K]. 
\end{align}
This, together with \cite[Lemma 6]{arora2018convergence} and \eqref{eq:delta}, implies $\sigma_1(\mb W_l) \le \left(\sqrt{2n}\right)^{1/L}$. Using  $\bW_{l+1}^T\bW_{l+1} = \bW_l\bW_l^T$ for all $l \in [L-2]$ in \eqref{eq:bala}, we obtain 
\begin{align*}
    \|\mb W_{L:1}^\top\mb W_{L:1}  - \left(\mb W_1^\top \mb W_1\right)^L \|_F & = \|\mb W_{L:1}^\top\mb W_{L:1}  - \mb W_{L-1:1}^\top \bW_{L-1}\bW_{L-1}^\top \bW_{L-1:1} \|_F \\
    & = \|\mb W_{L-1:1}^\top \left(\bW_L^\top\bW_L -  \bW_{L-1}\bW_{L-1}^\top\right) \bW_{L-1:1} \|_F \\
    & \le \|\bW_L^\top\bW_L -  \bW_{L-1}\bW_{L-1}^\top\|_F \prod_{l=1}^{L-1} \|\mb W_l\|^2 \\
    & \le \delta \left( 2n \right)^{(L-1)/L} \le \frac{n}{15L^2},
\end{align*}
where the second inequality uses \eqref{eq:bala} and $\|\mb W_l\| \le (\sqrt{2n})^{1/L}$ for all $l \in [L-1]$, and the last inequality follows from \eqref{eq:delta}. Using \eqref{eq1:lem bound W1} and Weyl's inequality, we obtain 
\begin{align*}
    \sigma_K\left(\left(\mb W_1^\top \mb W_1\right)^L\right) & \ge \sigma_K\left(\mb W_{L:1}^\top\mb W_{L:1}\right) - \|\mb W_{L:1}^\top\mb W_{L:1}  - \left(\mb W_1^\top \mb W_1\right)^L \|\\ 
    & \ge n - \|\mb W_{L:1}^\top\mb W_{L:1}  - \left(\mb W_1^\top \mb W_1\right)^L \|_F  \ge \left( 1 - \frac{1}{15L^2}\right) n. 
\end{align*}
Therefore, we have $\sigma_K(\bW_1) \ge \left( \left(1 - \frac{1}{15L^2}\right)n \right)^{1/2L}$. Using this and \eqref{eq:bala}, we obtain $\sigma_K(\bW_l) \ge \left( \left(1 - \frac{1}{15L^2}\right)n\right)^{1/2L}$ for all $l \in [L-1]$. This, together with Weyl's inequality and $\|\bW_{L}^T\bW_{L} - \bW_{L-1}\bW_{L-1}^T\|_F \le \delta$ in \eqref{eq:bala}, yields
\begin{align*}
   \sigma_K\left( \bW_{L}^T\bW_{L} \right) & \ge  \sigma_K\left( \bW_{L-1}^T\bW_{L-1} \right) -\|\bW_{L}^T\bW_{L} - \bW_{L-1}\bW_{L-1}^T\| \\
   & \ge \left( 1 - \frac{1}{15L^2}  \right)^{1/L}n^{1/L} - \delta \ge \left( \frac{n}{2} \right)^{1/L},   
\end{align*}
where the last inequality follows from 
\begin{align*}
    1 - \frac{1}{15L^2}   \ge  \left( \frac{1}{30L^2} + \left(\frac{1}{2}\right)^{1/L}  \right)^{L},\ \forall L \ge 1.
\end{align*}
Since $1-1/15L^2 \ge 1/2$ for all $L \ge 1$, we also have $\sigma_K(\bW_l) \ge \left( \frac{n}{2} \right)^{1/2L}$. Combining all this together yields \eqref{eq:bound Wl}.  
\end{proof}

Then, we show that if the weight matrices $\{\mb W_l\}_{l=1}^L$ satisfy \eqref{eq:opti} and \eqref{eq:bala}, the right singular vectors of $\mb W_{l}$ are equal or close to the left singular vectors of $\mb W_{l+1}$ for all $l \in [L-1]$. 

\begin{lemma}\label{lem:UV}
Suppose that  $\{\mb W_l\}_{l=1}^L$ satisfy \eqref{eq:opti} and \eqref{eq:bala} with $\delta$ satisfying \eqref{eq:delta}, and admit the SVD in \eqref{eq:SVD Wl} and \eqref{eq:SVD WL}. Then, it holds that
\begin{align}
   & \mb\Sigma_{l+1} = \mb\Sigma_{l},\ \|\bm{\Sigma}_{l,2}^2\|_F \le \delta,\ \forall l \in [L-1],\  \mb V_{l+1}^\top   \mb U_{l} = \bm{I}_d,\ \forall l \in [L-2],\label{eq:UV l} \\
   & \|\mb V_{L,1}^T \mb U_{L-1,2}\|_F \le  \frac{2\sqrt{\delta}\sqrt[4]{K}}{  n^{1/2L}},\ \sigma_{\min}(\mb V_{L,1}^\top   \mb U_{L-1,1}) \ge 1 - \frac{2\sqrt{\delta}\sqrt[4]{K}}{  n^{1/2L}}. \label{eq:UV L}
\end{align}
\end{lemma}
\begin{proof}
    It follows from $\mb W_{l+1}^\top\mb W_{l+1} = \mb W_{l}\mb W_{l}^\top$ for all $ l \in [L-2]$ and \eqref{eq:SVD Wl} that
    \begin{align*}
        \mb V_{l+1} \mb \Sigma_{l+1}^\top \mb \Sigma_{l+1} \mb V_{l+1}^\top = \mb U_{l} \mb \Sigma_{l}\mb \Sigma_{l}^\top  \mb U_{l}^\top.
    \end{align*}
    This, together with the argument in \cite[Proof of Theorem 1]{arora2018optimization}, implies \eqref{eq:UV l}. Using $\|\mb W_{L}^\top\mb W_{L} - \mb W_{L-1}\mb W_{L-1}^\top\|_F \le \delta$, \eqref{eq:SVD Wl}, and \eqref{eq:SVD WL}, we have
    \begin{align}
        \delta & \ge \|\mb V_{L} \mb \Sigma_{L}^\top \mb \Sigma_{L}\mb V_{L}^\top -  \mb U_{L-1} \mb \Sigma_{L-1}\mb \Sigma_{L-1}^\top  \mb U_{L-1}^\top \|_F \notag\\
        & = \|\bV_{L,1}\bm{\Sigma}^2_{L,1}\bV_{L,1}^T  - \bU_{L-1,1}\bm{\Sigma}^2_{L-1,1}\bU_{L-1,1}^T - \bU_{L-1,2}\bm{\Sigma}_{L-1,2}^2\bU_{L-1,2}^T\|_F. \label{eq1:lem UV}
    \end{align}
    Obviously, we have $\|\bm{\Sigma}_{L}^T\bm{\Sigma}_{L} - \bV_{L}^T\bU_{L-1}\bm{\Sigma}_{L-1}^2\bU_{L-1}^T\bV_{L}\|_F \le \delta$ due to $\mb V_L \in \calO^d$. This, together with  \cite[Lemma 4]{arora2018convergence}, implies 
    \begin{align*}
    \|\bm{\Sigma}_{L}^T\bm{\Sigma}_{L}  - \bm{\Sigma}_{L-1}^2\|_F \le \delta. 
    \end{align*}
    This, together with \eqref{eq:UV l}, implies $\|\mb \Sigma_{l,2}^2\|_F \le \delta$ for all $l \in [L-1]$. 
    Using this and the structures of $\mb \Sigma_L$ and $\mb \Sigma_{L-1}$, we obtain
    \begin{align}\label{eq2:lem UV}
        \|\bm{\Sigma}_{L,1}^2  - \bm{\Sigma}_{L-1,1}^2\|_F \le \delta,\ \|\bm{\Sigma}_{L-1,2}^2\|_F \le \delta. 
    \end{align}
    Using \eqref{eq1:lem UV}, $\|\mb U^\top \mb A \mb U\|_F \le \|\mb A \|_F$ for any $\mb U \in \calO^{d\times (d-K)}$, we further obtain  
    \begin{align*}
    \|\bU_{L-1,2}^T\bV_{L,1}\bm{\Sigma}^2_{L,1}\bV_{L,1}^T\bU_{L-1,2}   -  \bm{\Sigma}_{L-1,2}^2 \|_F \le \delta. 
    \end{align*} 
    This, together with \eqref{eq2:lem UV}, yields
    \begin{align*}
        \|\bU_{L-1,2}^T\bV_{L,1}\bm{\Sigma}^2_{L,1}\bV_{L,1}^T\bU_{L-1,2}\|_F \le 2\delta.   
    \end{align*}
    It follows from this and Lemma \ref{lem:singular AA} that
    \begin{align}\label{eq3:lem UV}
    \|\bm{\Sigma}_{L,1}\bV_{L,1}^T\bU_{L-1,2}\|_F \le \sqrt{2\delta}\sqrt[4]{K}. 
    \end{align}
    Noting that $\|\bm{\Sigma}_{L,1}\bV_{L,1}^T\bU_{L-1,2}\|_F \ge \sigma_{\min}(\bm{\Sigma}_{L,1})\|\bV_{L,1}^T\bU_{L-1,2}\|_F$, we have 
    \begin{align}\label{eq4:lem UV}
        \|\bV_{L,1}^T\bU_{L-1,2}\|_F  \le \frac{\|\bm{\Sigma}_{L,1}\bV_{L,1}^T\bU_{L-1,2}\|_F}{\sigma_{\min}(\bm{\Sigma}_{L,1})} \le \frac{\sqrt{2\delta}\sqrt[4]{K}}{\left( \frac{n}{2} \right)^{1/2L}} \le  \frac{2\sqrt{\delta}\sqrt[4]{K}}{  n^{1/2L}},
    \end{align}
    where the second inequality follows from \eqref{eq3:lem UV} and Lemma \ref{lem:bound Wl}. Then, we compute
\begin{align*}
\sigma_{\min}^2(\bV_{L,1}^T\bU_{L-1,1}) & =  \min_{\|\bx\|=1} \|\bU_{L-1,1}^T\bV_{L,1}\bx\|^2  =    \min_{\|\bx\|=1} \bx^T \bV_{L,1}^T\bU_{L-1,1}\bU_{L-1,1}^T\bV_{L,1}\bx  \\
& = \min_{\|\bx\|=1} \bx^T \bV_{L,1}^T\left(\bI - \bU_{L-1,2}\bU_{L-1,2}^T\right)\bV_{L,1}\bx = 1 - \max_{\|\bx\|=1} \|\bU_{L-1,2}^T\bV_{L,1}\bx\|^2 \\
& \ge 1 - \sigma_{\max}^2(\bU_{L-1,2}^T\bV_{L,1}) \ge 1 - \|\bU_{L-1,2}^T\bV_{L,1}\|_F^2,
\end{align*}
where the third equality follows from $ \bU_{L-1} \in \calO^d$. This, together with \eqref{eq4:lem UV}, implies
\begin{align*}
\sigma_{\min}(\bV_{L,1}^T\bU_{L-1,1}) \ge \sqrt{ 1 - \|\bU_{L-1,2}^T\bV_{L,1}\|_F^2} \ge 1 - \|\bU_{L-1,2}^T\bV_{L,1}\|_F \ge 1 -  \frac{4\delta\sqrt{K}}{  n^{1/2L}}. 
\end{align*}
\end{proof}

Recall that $\bar{\bx}_k = \sum_{i=1}^n \bm{x}_{k,i} / n$ and $\bar{\bx} = \sum_{k=1}^K \bar{\bm{x}}_{k} / K$. Let 
\begin{align} 
& \bm{\Delta}_W = \left[\bm{\delta}_{1,1},\dots,\bm{\delta}_{1,n},\dots,\bm{\delta}_{K,1},\dots,\bm{\delta}_{K,n} \right] \in \R^{d\times N},\ \text{where}\ \bm{\delta}_{k,i} = \bm{x}_{k,i}-\bar{\bx}_k,\ \forall k,i, \label{eq:Delta W} \\
& \bm{\Delta}_B = \left[\bar{\bm{\delta}}_1,\dots,\bar{\bm{\delta}}_K \right] \in \R^{d\times K},\ \text{where}\ \bar{\bm{\delta}}_k = \bar{\bx}_k - \bar{\bx},\ \forall k \in [K]. \label{eq:Delta B} 
\end{align}
This, together with \eqref{eq:zl}, \eqref{eq:Sigma}, and \eqref{eq:D measure}, yields 
\begin{align}\label{eq:sepa measure}
D_0 = \frac{K}{N}\frac{\|\bm{\Delta}_W\|_F^2}{\|\bm{\Delta}_B\|_F^2},\quad D_l  = \frac{K}{N}\frac{\|\bW_{l:1}\bm{\Delta}_W\|_F^2}{\|\bW_{l:1}\bm{\Delta}_B\|_F^2}. 
\end{align} 

\begin{proof}[Proof of Theorem \ref{thm:NC}]
Using \eqref{eq:UV l} in \Cref{lem:UV}, there exists a diagonal matrix $\tilde{\bm{\Sigma}} = \mathrm{diag}(\sigma_1,\dots,\sigma_d)$ with $\sigma_1\ge\dots\ge\sigma_d$ such that  
\begin{align}\label{eq1:thm 1}
  \bm{\Sigma}_l = \tilde{\bm{\Sigma}},\ \forall l \in [L-1]. 
\end{align} 
According to \Cref{lem:bound Wl}, we have $\sigma_K(\mb W_l) \ge \left(\sqrt{n/2} \right)^{1/L} \ge \delta$, where the last inequality follows from \eqref{eq:delta} and  \eqref{eq:epsilon}. It follows from this and \eqref{eq:d-2K} that $\calA \subseteq \{K+1,\dots,d\}$. For simplicity, we write  $\tilde{\bm{\Sigma}} = \begin{bmatrix}
\tilde{\bm{\Sigma}}_1 & \bm{0} \\
\bm{0} & \tilde{\bm{\Sigma}}_2
\end{bmatrix}$ satisfying $\tilde{\bm{\Sigma}}_1 = \diag\left(\sigma_1,\dots,\sigma_K \right)$ and $\tilde{\bm{\Sigma}}_2 = \diag\left(\sigma_{K+1},\dots,\sigma_d \right)$. Using \eqref{eq:UV L} and \eqref{eq:d-2K}, we have 
\begin{align*}
    \sum_{i=K+1}^d \sigma_i^4 \le \delta^2 = \varepsilon^4(d-K).
\end{align*}
Using \eqref{eq:d-2K} and letting $\calA^c = \{K+1,\dots,d\}\setminus \calA$, we have 
\begin{align*}
    \sum_{i \in \calA^c} \sigma_i^4 + (d-2K)\varepsilon^4 \le \varepsilon^4(d-K),
\end{align*}
which implies $\sigma_i \le \sqrt[4]{K}\varepsilon$ for all $i \in \calA^c$. Therefore, we have
\begin{align}\label{eq4:thm 1}
    \sigma_i \le \sqrt[4]{K}\varepsilon,\ i = K+1,\dots,d. 
\end{align}
It follows from \eqref{eq:UV l} in \Cref{lem:UV} that 
\begin{align}\label{eq2:thm 1}
\bW_l\dots\bW_1 = \bU_{l}\tilde{\bm{\Sigma}}^l\bV_1^T,\ \forall l \in [L-1]. 
\end{align}
It follows from \eqref{eq:opti}, \eqref{eq:Delta W}, and $\bY = \bI_K\otimes \bm{1}^T_n$  that
\begin{align}\label{eq3:thm 1}
\bW_{L:1}\bm{\Delta}_W & =  \bW_{L:1}\begin{bmatrix}
   \bm{x}_{1,1} &  \dots & \bm{x}_{1,n} & \dots  &  \bm{x}_{K,1} & \dots & \bm{x}_{K,n}  \end{bmatrix} \notag \\
   &\quad -  \bW_{L:1} \begin{bmatrix} \bar{\bm{x}}_1 & \dots & \bar{\bm{x}}_1 & \dots & \bar{\bm{x}}_K & \dots & \bar{\bm{x}}_K \end{bmatrix} = \bm{0}. 
\end{align}
Substituting \eqref{eq:SVD Wl}, \eqref{eq:UV l}, and \eqref{eq1:thm 1} into \eqref{eq3:thm 1} yields
\begin{align*}
\bm{0} = \bW_L\bU_{L-1}\tilde{\bm{\Sigma}}^{L-1}\bV_1^T\bm{\Delta}_W = \bW_L\left(\bU_{L-1,1}\tilde{\bm{\Sigma}}_1^{L-1}\bV_{1,1}^T + \bU_{L-1,2}\tilde{\bm{\Sigma}}_2^{L-1}\bV_{1,2}^T\right)\bm{\Delta}_W. 
\end{align*}
As a result, we obtain for all $l \in [L-1]$,
\begin{align} \label{eq5:thm 1}
\|\bW_L \bU_{L-1,1}\tilde{\bm{\Sigma}}_1^{L-1}\bV_{1,1}^T\bm{\Delta}_W \|_F & = \|\mb U_L\mb \Sigma_{L,1} \mb V_{L,1}^T\bU_{L-1,2}\tilde{\bm{\Sigma}}_2^{L-1}\bV_{1,2}^T\bm{\Delta}_W\|_F \notag \\
& \le \|\bm{\Sigma}_{L,1}\|\|\bm{V}_{L,1}^T\bU_{L-1,2}\|\|\tilde{\bm{\Sigma}}_2^{L-1-l}\|\|\tilde{\bm{\Sigma}}_2^l\bV_{1,2}^T\bm{\Delta}_W\|_F \notag \\
& \le (2n)^{1/2L} \frac{2\sqrt{\delta}\sqrt[4]{K}}{  n^{1/2L}}\left( \sqrt[4]{K}\varepsilon \right)^{L-1-l}\|\tilde{\bm{\Sigma}}_2^l\bV_{1,2}^T\bm{\Delta}_W\|_F \notag \\
& \le 2\sqrt{2}K^{(L-l)/4}\sqrt[4]{d-K}\varepsilon^{L-l}\|\tilde{\bm{\Sigma}}_2^l\bV_{1,2}^T\bm{\Delta}_W\|_F,
\end{align}
where the second inequality follows from \eqref{eq:bound Wl}, \eqref{eq:UV L}, and \eqref{eq4:thm 1}, and the last inequality uses \eqref{eq:delta}.  
Next, we compute
\begin{align}\label{eq6:thm 1}
\sigma_{\min}(\bW_L \bU_{L-1,1}) & = \sigma_{\min}( \bU_L\bm{\Sigma}_{L,1}\bV_{L,1}^T\bU_{L-1,1} ) = \sigma_{\min}(\bm{\Sigma}_{L,1}\bV_{L,1}^T\bU_{L-1,1} ) \notag \\
& \ge \sigma_{\min}(\bV_{L,1}^T\bU_{L-1,1} ) \sigma_{\min}(\bm{\Sigma}_{L,1})  \ge  \left(1 - \frac{2\sqrt{\delta}\sqrt[4]{K}}{n^{1/2L}}\right)\left( \frac{n}{2}\right)^{1/2L} \notag \\
& \ge \frac{1}{2} \left( \frac{n}{2}\right)^{1/2L},
\end{align}
where the second inequality follows from \eqref{eq:bound Wl} and \eqref{eq:UV L}, and the last inequality uses \eqref{eq:delta} and \eqref{eq:epsilon}. Then, we compute 
\begin{align}
\|\bW_L \bU_{L-1,1} \tilde{\bm{\Sigma}}_1^{L-1}\bV_{1,1}^T\bm{\Delta}_W \|_F & \ge \sigma_{\min}(\bW_L \bU_{L-1,1}\tilde{\bm{\Sigma}}_1^{L-l-2})\|\tilde{\bm{\Sigma}}_1^{l+1}\bV_{1,1}^T\bm{\Delta}_W \|_F \notag \\
& \ge \sigma_{\min}(\bW_L \bU_{L-1,1})\sigma_{\min}(\tilde{\bm{\Sigma}}_1^{L-l-2})\|\tilde{\bm{\Sigma}}_1^{l+1}\bV_{1,1}^T\bm{\Delta}_W \|_F  \notag \\
& \ge   \frac{1}{2}\left( \frac{n}{2} \right)^{(L-l-1)/2L} \|\tilde{\bm{\Sigma}}_1^{l+1}\bV_{1,1}^T\bm{\Delta}_W \|_F,
\end{align}
where the last inequality follows from \eqref{eq:bound Wl} and \eqref{eq6:thm 1}. This, together with \eqref{eq5:thm 1}, implies for $l \in [L-2]$,
\begin{align}\label{eq:kappa}
\frac{\|\tilde{\bm{\Sigma}}_1^{l+1}\bV_{1,1}^T\bm{\Delta}_W \|_F}{\|\tilde{\bm{\Sigma}}_2^{l}\bV_{1,2}^T\bm{\Delta}_W \|_F} \le \kappa_1\varepsilon^{L-l},\ \text{where}\ \kappa_1 := \frac{4\sqrt{2}K^{(L-l)/4}\sqrt[4]{d-K}}{(n/2)^{(L-l-1)/2L}}. 
\end{align}
It follows from \eqref{eq:SVD Wl}, \eqref{eq1:thm 1}, and \eqref{eq2:thm 1} that for all $l \in [L-1]$,
\begin{align*}
\|\bW_{l:1}\bm{\Delta}_W\|_F^2 = \|\bU_l\tilde{\bm{\Sigma}}^l\bm{V}_1^T\bm{\Delta}_W\|_F^2 =   \|\tilde{\bm{\Sigma}}^{l}_1\bm{V}_{1,1}^T\bm{\Delta}_W\|_F^2 +  \|\tilde{\bm{\Sigma}}^{l}_2\bm{V}_{1,2}^T\bm{\Delta}_W\|_F^2.
\end{align*}
This, together with \eqref{eq4:thm 1}, yields 
\begin{align}\label{eq8:thm 1}
    \frac{\|\mb W_{l+1:1}\bm{\Delta}_W\|_F^2}{\|\mb W_{l:1}\bm{\Delta}_W\|_F^2} & = \frac{\|\tilde{\bm{\Sigma}}^{l+1}_1\bm{V}_{1,1}^T\bm{\Delta}_W\|_F^2 +  \|\tilde{\bm{\Sigma}}^{l+1}_2\bm{V}_{1,2}^T\bm{\Delta}_W\|_F^2}{\|\tilde{\bm{\Sigma}}^{l}_1\bm{V}_{1,1}^T\bm{\Delta}_W\|_F^2 +  \|\tilde{\bm{\Sigma}}^{l}_2\bm{V}_{1,2}^T\bm{\Delta}_W\|_F^2} \notag \\
    & \le \frac{ \kappa_1\varepsilon^{L-l}\|\tilde{\bm{\Sigma}}^{l}_2\bm{V}_{1,2}^T\bm{\Delta}_W\|_F +  \sqrt{K}\varepsilon^2\|\tilde{\bm{\Sigma}}^{l}_2\bm{V}_{1,2}^T\bm{\Delta}_W\|_F^2}{\|\tilde{\bm{\Sigma}}^{l}_2\bm{V}_{1,2}^T\bm{\Delta}_W\|_F^2}  \notag \\
    & \le \kappa_1\varepsilon^{L-l} + \sqrt{K}\varepsilon^2 \le (\sqrt{K}+1)\varepsilon^2,
\end{align}
where the first inequality follows from  \eqref{eq4:thm 1} and \eqref{eq:kappa}.  
Suppose that we have 
\begin{align}\label{eq9:thm 1}
    \| \bm{V}_{1,2}^T\bm{\Delta}_B \|_F \le  \|\bm{V}_{1,1}^T\bm{\Delta}_B\|_F. 
\end{align}
Then, we have 
\begin{align*}
    \|\tilde{\bm{\Sigma}}^{l+1}_2\bm{V}_{1,2}^T\bm{\Delta}_B\|_F \le \| \bm{V}_{1,2}^T\bm{\Delta}_B\|_F \le \|\bm{V}_{1,1}^T\bm{\Delta}_B\|_F \le \|\tilde{\bm{\Sigma}}^{l+1}_1\bm{V}_{1,2}^T\bm{\Delta}_B\|_F, 
\end{align*}
where the first inequality follows from $\sigma_i \le 1$ for all $i=K+1,\dots,d$ due to \eqref{eq4:thm 1} and \eqref{eq:epsilon}, the second inequality uses \eqref{eq9:thm 1}, and the last inequality holds by \eqref{eq:bound Wl}.   
Using \eqref{eq2:thm 1} and a similar argument as above, we compute
\begin{align*}
    \frac{\|\mb W_{l:1}\bm{\Delta}_B\|_F^2}{\|\mb W_{l+1:1}\bm{\Delta}_B\|_F^2}  & = \frac{\|\tilde{\bm{\Sigma}}^{l}_1\bm{V}_{1,1}^T\bm{\Delta}_B\|_F^2 +  \|\tilde{\bm{\Sigma}}^{l}_2\bm{V}_{1,2}^T\bm{\Delta}_B\|_F^2}{\|\tilde{\bm{\Sigma}}^{l+1}_1\bm{V}_{1,1}^T\bm{\Delta}_B\|_F^2 +  \|\tilde{\bm{\Sigma}}^{l+1}_2\bm{V}_{1,2}^T\bm{\Delta}_B\|_F^2} \\
    & \le \frac{\|\tilde{\bm{\Sigma}}^{l}_1\bm{V}_{1,1}^T\bm{\Delta}_B\|_F^2 +  \|\tilde{\bm{\Sigma}}^{l}_2\bm{V}_{1,2}^T\bm{\Delta}_B\|_F^2}{\|\tilde{\bm{\Sigma}}^{l+1}_1\bm{V}_{1,1}^T\bm{\Delta}_B\|_F^2}  \le \frac{2\|\tilde{\bm{\Sigma}}^{l}_1\bm{V}_{1,1}^T\bm{\Delta}_B\|_F^2 }{\|\tilde{\bm{\Sigma}}^{l+1}_1\bm{V}_{1,1}^T\bm{\Delta}_B\|_F^2} \\
    & \le \frac{2\|\tilde{\bm{\Sigma}}^{l}_1\bm{V}_{1,1}^T\bm{\Delta}_B\|_F^2 }{\sigma_{\min}(\tilde{\bm{\Sigma}})\|\tilde{\bm{\Sigma}}^{l}_1\bm{V}_{1,1}^T\bm{\Delta}_B\|_F^2}  \le 2,
\end{align*}
where the last inequality uses $\sigma_{\min}(\tilde{\bm{\Sigma}}) \ge 1$ due to \eqref{eq:bound Wl}. This, together with \eqref{eq:sepa measure} and \eqref{eq8:thm 1}, implies \eqref{eq:linear decay}. 

The rest of the proof is devoted to showing \eqref{eq9:thm 1}. For ease of exposition, let $\bm{V} := \bm{X}\bm{Y}^T/\sqrt{n} \in \R^{d\times K}$. One can verify $\bm{V}^T\bm{V} = \bm{I}_K$. In addition, we can compute
\begin{align}\label{eq10:thm 1}
    \bm{V}\bm{V}^T \bm{\Delta}_B =  \bm{\Delta}_B,  
\end{align}
where the equality follows from \eqref{eq:Delta B} and $\mb X \in \mathcal{O}^d$. Using \eqref{eq:opti} and $\mb X \in \mathcal{O}^d$, we have
\begin{align}
    \mb W_{L:1} = \mb Y \mb X^T = \sqrt{n}\mb V^T.  
\end{align}
This, together with \eqref{eq:SVD WL} and \eqref{eq2:thm 1}, yields
\begin{align*}
   \bU_{L}\bm{\Sigma}_{L,1}\bV_{L,1}^T \bU_{L-1}\tilde{\bm{\Sigma}}^{L-1}\bV_1^T = \sqrt{n}\mb V^T. 
\end{align*}
Therefore, we obtain
\begin{align*}
 n\mb V\mb V^T & = \mb V_1 \tilde{\bm{\Sigma}}^{L-1} \bU_{L-1}^T\bV_{L,1} \bm{\Sigma}_{L,1}^2   \bV_{L,1}^T \bU_{L-1}\tilde{\bm{\Sigma}}^{L-1}\bV_1^T \\
 & = \left( \mb V_{1,1} \tilde{\bm{\Sigma}}_1^{L-1} \bU_{L-1,1}^T + \mb V_{1,2} \tilde{\bm{\Sigma}}_2^{L-1} \bU_{L-1,2}^T \right) \bV_{L,1} \bm{\Sigma}_{L,1}^2   \bV_{L,1}^T \left(  \bU_{L-1,1} \tilde{\bm{\Sigma}}_1^{L-1}\mb V_{1,1}^T + \mb U_{L-1,2} \tilde{\bm{\Sigma}}_2^{L-1} \bV_{1,2}^T \right). 
\end{align*}
Using \eqref{eq:UV L} and Davis-Kahan Theorem \cite[Theorem V.3.6]{stewart1990matrix}, we have 
\begin{align}\label{eq11:thm 1}
    \|\bm{V}\bm{V}^T - \bm{V}_{1,1}\bm{V}_{1,1}^T\|_F \le \frac{1}{2}. 
\end{align}
Then, we compute
\begin{align*}
    \|\bm{V}_{1,2}^T\bm{\Delta}_B\|_F & =  \|(\bm{I} - \bm{V}_{1,1}\bm{V}_{1,1}^T)\bm{\Delta}_B \|_F = \|(\bm{I} - \bm{V}\bm{V}^T) \bm{\Delta}_B + (\bm{V}\bm{V}^T - \bm{V}_{1,1}\bm{V}_{1,1}^T) \bm{\Delta}_B  \|_F  \\
    & =  \|(\bm{V}\bm{V}^T - \bm{V}_{1,1}\bm{V}_{1,1}^T) \bm{\Delta}_B  \|_F \le \frac{1}{2}\|\bm{\Delta}_B\|_F. 
\end{align*}
where the third equality uses \eqref{eq10:thm 1}, and the inequality follows from \eqref{eq11:thm 1}. This, together with $\mb V_1 \in {\cal O}^d$, directly implies \eqref{eq9:thm 1}. 
\end{proof}

\section{Auxiliary Results} 

\begin{lemma}\label{lem:singular AA}
Given a matrix $\bA \in \R^{m\times n}$ of rank $r \ge 0$, we have
\begin{align}\label{eq:singular AA}
 \|\bA^T\bA\|_F \le \|\bA\|_F^2 \le \sqrt{r}\|\bA^T\bA\|_F. 
\end{align}
\end{lemma}
\begin{proof}
Let $\bA = \bU\bm{\Sigma}\bV^T$ be a singular value decomposition of $\bA$, where
\begin{align*}
\bm{\Sigma} = \begin{bmatrix}
\tilde{\bm{\Sigma}} & \bm{0} \\
\bm{0} & \bm{0} 
\end{bmatrix},\ \tilde{\bm{\Sigma}} = \diag(\sigma_1,\dots,\sigma_r),
\end{align*} 
$\bU \in \mO^m$, and $\bV \in \mO^n$. Then we compute 
\begin{align*}
 \|\bA^T\bA\|_F^2 = \|\bm{\Sigma}^T\bm{\Sigma}\|_F^2 =  \sum_{i=1}^r \sigma_i^4,\ \|\bA\|_F^2 = \|\bm{\Sigma}\|_F^2 = \sum_{i=1}^r \sigma_i^2,
\end{align*} 
which, together with the AM-QM inequality, directly implies \eqref{eq:singular AA}. 
\end{proof}